\documentclass{article}

\usepackage{arxiv}

\usepackage{amsmath,amsfonts,bm}
\usepackage{xcolor}

\def\eqref#1{equation~\ref{#1}}
\def\1{\bm{1}}

\DeclareMathAlphabet{\mathsfit}{\encodingdefault}{\sfdefault}{m}{sl}
\SetMathAlphabet{\mathsfit}{bold}{\encodingdefault}{\sfdefault}{bx}{n}

\DeclareMathOperator*{\argmin}{arg\,min}

\usepackage{amsmath,epsfig}
\usepackage{amssymb}
\usepackage{amsthm}
\usepackage{algorithm}
\usepackage{algpseudocode}
\usepackage{tikz}
\usepackage{tabu}
\usepackage{graphicx}
\usepackage{subfig}
\usepackage{epstopdf}
\usepackage{epsfig}
\usepackage{booktabs}
\usepackage[flushleft]{threeparttable}
\usetikzlibrary{tikzmark,calc}
\usepackage{float}
\usepackage{multirow}
\usepackage{hyperref}
\usepackage{verbatim}
\usepackage{wrapfig}
\usepackage{rotating}
\usepackage{mathtools}
\usepackage{bm}
\usepackage{enumitem}
\usepackage{standalone}
\usepackage{pgfplots}
\usepackage{pgfplotstable}
\numberwithin{equation}{section} 
\newtheorem{theorem}{Theorem}[section]                   %[section]   %numberes automatically
\newtheorem{definition}[theorem]{Definition}

\usepackage{microtype}
\usepackage{flushend}
\newcommand\Algphase[1]{%
\vspace*{-.7\baselineskip}\Statex\hspace*{\dimexpr-\algorithmicindent-2pt\relax}\rule{\columnwidth}{0.4pt}%
\Statex\hspace*{-\algorithmicindent}\textbf{#1}%
\vspace*{-.7\baselineskip}\Statex\hspace*{\dimexpr-\algorithmicindent-2pt\relax}\rule{\columnwidth}{0.4pt}%
}
\pgfdeclarelayer{background}
\pgfdeclarelayer{foreground}
\pgfsetlayers{background,main,foreground}

\tikzstyle{sensor}=[draw, fill=blue!20, text width=4em, 
    text centered, minimum height=2.5em]
\tikzstyle{normalCell}=[draw, fill=orange!20, text width=3.2em, 
    text centered, minimum width=2.0em, minimum height=1.4em, 
    rounded corners]
\tikzstyle{conv}=[draw, fill=blue!20, text width=4.2em, 
    text centered, minimum width=2.0em, minimum height=1.4em, rounded corners]
\tikzstyle{operation}=[draw, fill=white!20, text width=3.2em, 
    text centered, minimum width=2.0em, minimum height=1.4em, rounded corners]
\tikzstyle{cell_in}=[draw, fill=green!20, text width=3.0em, 
    text centered, minimum width=2.5em, minimum height=2.3em]
\tikzstyle{cell_out}=[draw, fill=yellow!20, text width=3.0em, 
    text centered, minimum width=2.5em, minimum height=2.3em]
\tikzstyle{square}=[draw, fill=blue!20, minimum size=2em]
\tikzstyle{choice}=[draw, fill=green!20, text width=3.2em, 
    text centered, minimum width=2.0em, minimum height=1.4em, rounded corners]
\tikzstyle{customNode}=[draw, fill=purple!25, text width=3.2em, 
    text centered, minimum width=2.0em, minimum height=1.4em, rounded corners]
\tikzstyle{ann} = [above, text width=5em, text centered]
\tikzstyle{wa} = [sensor, text width=10em, fill=red!20, 
    minimum height=6em, rounded corners, drop shadow]
\tikzstyle{sc} = [sensor, text width=13em, fill=red!20, 
    minimum height=10em, rounded corners, drop shadow]
\tikzstyle{texto} = [above, text width=6em, text centered]
\tikzstyle{oneShotModel}=[draw, fill=white!10, text width=3.0em, 
    text centered, minimum width=2.0em, minimum height=6.0em]
\tikzstyle{sampleNode}=[draw, fill=white!10, text width=5.0em, 
    text centered, minimum width=2.0em, minimum height=1.0em]
\tikzstyle{CSNode}=[draw, fill=white!10, text width=4.0em, 
    text centered, minimum width=2.0em, minimum height=6.0em]
\tikzstyle{restrictionNode}=[draw, fill=white!10, text width=5.0em, 
    text centered, minimum width=2.0em, minimum height=1.0em]

\newcommand{\backgrounda}[6]{%
  \begin{pgfonlayer}{background}
    \path (#1.west |- #2.north)+(-0.2,1.1) node (a1) {};
    \path (#3.east |- #4.south)+(+0.2,-0.25) node (a2) {};
    \path[fill=#5!20,rounded corners, draw=black!50, dashed]
      (a1) rectangle (a2);
    \path (a1.west |- a1.north)+(3.6, -0.8) node (u1)[texto]
      {\large #6};
  \end{pgfonlayer}}
  
\newcommand{\backgroundb}[6]{%
  \begin{pgfonlayer}{background}
    \path (#1.west |- #2.north)+(-0.5,1.157) node (a1) {};
    \path (#3.east |- #4.south)+(+0.5,-0.25) node (a2) {};
    \path[fill=#5!20,rounded corners, draw=black!50, dashed]
      (a1) rectangle (a2);
    \path (a1.west |- a1.north)+(5.5, -0.8) node (u1)[texto]
      {\large #6};
  \end{pgfonlayer}}
  
 \newcommand{\backgroundc}[6]{%
  \begin{pgfonlayer}{background}
    \path (#1.west |- #2.north)+(-0.3, 1.28) node (a1) {};
    \path (#3.east |- #4.south)+(+0.3,-1.42) node (a2) {};
    \path[fill=#5!20,rounded corners, draw=black!50, dashed]
      (a1) rectangle (a2);
    \path (a1.west |- a1.north)+(1.1, -0.65) node (u1)[texto]
      {\small #6};
  \end{pgfonlayer}}
  
\newcommand{\backgroundd}[6]{%
  \begin{pgfonlayer}{background}
    \path (#1.west |- #2.north)+(-0.3, 0.5) node (a1) {};
    \path (#3.east |- #4.south)+(+0.3,-0.5) node (a2) {};
    \path[fill=#5!20,rounded corners, draw=black!50, dashed]
      (a1) rectangle (a2);
    \path (a1.west |- a1.north)+(2.3, -0.65) node (u1)[texto]
      {\small #6};
  \end{pgfonlayer}}

\begin{document}
\title{Hyperparameter Optimization in Neural Networks via Structured Sparse Recovery}

\author{%
  Minsu Cho\thanks{The authors would like to thank Amitangshu Mukherjee, Soumik Sarkar, and Alberto Speranzon for helpful discussions. This work was performed when the authors were at Iowa State University (Ames, IA), and were supported in part by NSF grants CCF-1566281, CAREER CCF-1750920/2005804, GPU gift grants from the NVIDIA Corporation, and a faculty fellowship from the Black and Veatch Foundation. Parts of this paper appeared as a conference publication in~\cite{cho2019reducing} and the Arxiv preprint~\cite{cho2019one}.} \\
  Tandon School of Engineering\\
  New York University\\
  \texttt{mc8065@nyu.edu} \\
  \And
  Mohammadreza Soltani \\ 
  ECE Department \\
  Duke University \\
  \texttt{mohammadreza.soltani@duke.edu} \\
  \And
  Chinmay Hegde \\
  Tandon School of Engineering\\
  New York University\\
  \texttt{chinmay.h@nyu.edu}\\
}

\maketitle

\begin{abstract}
In this paper, we study two important problems in the automated design of neural networks --- Hyper-parameter Optimization (HPO), and Neural Architecture Search (NAS) ---  through the lens of sparse recovery methods. 

In the first part of this paper, we establish a novel connection between HPO and structured sparse recovery. In particular, we show that a special encoding of the hyperparameter space enables a natural group-sparse recovery formulation, which when coupled with HyperBand (a multi-armed bandit strategy), leads to improvement over existing hyperparameter optimization methods. Experimental results on image datasets such as CIFAR-10 confirm the benefits of our approach.
% In particular, we show that a special encoding of hyperparameter space enables a natural group-sparse recovery formulation, which when coupled with HyperBand (a multi-armed bandit strategy) leads to improvement over existing hyperparameter optimization methods such as Successive Halving and Random Search. 

In the second part of this paper, we establish a connection between NAS and structured sparse recovery. Building upon ``one-shot'' approaches in NAS, we propose a novel algorithm that we call CoNAS by merging ideas from one-shot approaches with a techniques for learning low-degree sparse Boolean polynomials. We provide theoretical analysis on the number of validation error measurements. Finally, we validate our approach on several datasets and discover novel architectures hitherto unreported, achieving competitive (or better) results in both performance and search time compared to the existing NAS approaches.
% which remains to be a very challenging meta-learning problem. Several recent works (called ``one-shot'' approaches) have focused on dramatically reducing NAS running time by leveraging proxy models that still provide architectures with competitive performance. 
% In our work, we propose a new meta-learning algorithm that we call CoNAS, or Compressive sensing-based Neural Architecture Search. Our approach merges ideas from one-shot NAS approaches with iterative techniques for learning low-degree sparse Boolean polynomial functions.
% We validate our approach on several standard test datasets, discover novel architectures hitherto unreported, and achieve competitive (or better) results in both performance and search time compared to existing NAS approaches. Further, we provide theoretical analysis via upper bounds on the number of validation error measurements needed to perform reliable meta-learning; to our knowledge, these analysis tools are novel to the NAS literature and may be of independent interest.
\end{abstract}

% \begin{IEEEkeywords}
% Sparse recovery, deep learning, hyperparameter optimization, neural architecture search, Boolean functions.  
% \end{IEEEkeywords}

% \fi
% \IEEEpeerreviewmaketitle

\newcommand{\var}[1]{\text{\texttt{#1}}}
\newcommand{\func}[1]{\text{\textsl{#1}}}

\makeatletter
\newcounter{phase}[algorithm]
\newlength{\phaserulewidth}
\newcommand{\setphaserulewidth}{\setlength{\phaserulewidth}}
\newcommand{\phase}[1]{%
  %\vspace{-1.25ex}
  % Top phase rule
  \Statex\leavevmode\llap{\rule{\dimexpr\labelwidth+\labelsep}{\phaserulewidth}}\rule{\linewidth}{\phaserulewidth}
  \Statex\strut\refstepcounter{phase}\textbf{Stage~\thephase~--~#1}% Phase text
  % Bottom phase rule
  %\vspace{-1.25ex}
  \Statex\leavevmode\llap{\rule{\dimexpr\labelwidth+\labelsep}{\phaserulewidth}}\rule{\linewidth}{\phaserulewidth}}
\makeatother

\setphaserulewidth{.7pt}

\newlength{\whilewidth}
\settowidth{\whilewidth}{\algorithmicwhile\ }
\algdef{SE}[parWHILE]{parWhile}{EndparWhile}[1]
  {\parbox[t]{\dimexpr\linewidth-\algmargin}{%
     \hangindent\whilewidth\strut\algorithmicwhile\ #1\ \algorithmicdo\strut}}{\algorithmicend\ \algorithmicwhile}%
\algnewcommand{\parState}[1]{\State%
  \parbox[t]{\dimexpr\linewidth-\algmargin}{\strut #1\strut}}

\section{Introduction}

%\IEEEPARstart{H}{ello}
\subsection{Motivation}
% \IEEEPARstart{M}{achine} learning (ML) models have been developed successfully to perform complex prediction tasks in recent years. 
Despite the success of complex deep learning (DL) models in many data-driven tasks, these models often require a substantial manual effort (involving trial-and-error) for choosing a suitable set of hyperparameters and architectures such as learning rate, regularization coefficients, dropout ratio, filter sizes, and network size. Hyper-parameter optimization (HPO) addresses the problem of searching for suitable hyperparameters that solve a given machine learning problem. Furthermore, neural architecture search (NAS) methods seek to automatically construct a suitable architecture of neural networks with competitive (or better) results over hand-designed architectures with as small computational budget as possible. 

In this paper, we propose two novel methods to solve HPO and NAS problems by borrowing ideas from sparse recovery and {compressive sensing} (CS)~\cite{candes2006near, donoho2006compressed}. CS has received significant attention in both signal processing and statistics over the last decade, and has influenced the development of numerous advances in nonlinear and combinatorial optimization. Compressive sensing provides an alternative to the conventional sampling paradigm by efficiently recover a sparse signal either exactly or approximately from a small number of measurements. 
% The intuition behind CS is that if a signal can be represented via a few basis functions, then the corresponding coefficients can be estimated (either exactly or approximately) from a small number of measurements. 

In the context of HPO/NAS, the main challenge is to evaluate the test/validation performance of a (combinatorially) large number of hyperparameters/architectures candidates. To overcome this, our methods leverage sparse recovery techniques 
% (i.e., estimating the sparse coefficients of a signal from a limited number of measurements) 
to find an approximate, yet competitive, solution through fewer number of candidate performance evaluations (measurements). 

\subsection{Our Contributions}
% We approach the HPO and the NAS problem via the lens of compressive sensing. 
% The field of compressive sensing (or sparse recovery), introduced by the seminal works of (\cite{candes2006near}, \cite{donoho2006compressed}), has received significant attention in both signal processing literature and applications over the last decade, and has influenced the development of numerous advances in nonlinear and combinatorial optimization. 
Below, we describe our contributions for solving the HPO and NAS problems. The core to both of our solutions is an adaptation of techniques for learning low-degree sparse Boolean polynomial functions.

\subsubsection{Hyperparameter Optimization} We propose an extension to the Harmonica algorithm (\cite{hazan2017hyperparameter}), a spectral approach for recovering a sparse Boolean representation of an objective function relevant to the HPO problem. While Harmonica successfully finds important categorical hyperparameters, it does not excel in finding numerical, continuous hyperparameters (such as learning rate). We propose a new group-sparse representation on continuous hyperparameter values that reduces not only the dimension of the search space, but also groups the hyperparameters; this improves both accuracy and stability of HPO. We provide visualizations of the achieved approximate minima by our proposed algorithm in hyperparameter space and demonstrates its success for classification tasks.

% We validate our numerical expression with hyperparameters grouping to examine its guidance accurately. 
% show that this algorithm closely approximates the global minimum in hyperparameter space by plotting the loss surface with two hyperparameters. the robustness of our algorithm is illustrated by combining the guidance to the decision-theoretic methods with measurable improvements in test loss using a CNN architecture trained on the CIFAR-10 image classification dataset.
 
\subsubsection{Neural Architecture Search} We propose a new NAS algorithm called CoNAS (Compressive sensing-based Neural Architecture Search), which merges ideas from sparse recovery with the so-called ``one-shot'' architecture search methods~\cite{bender2018understanding}, \cite{li2019random} (Please see Section~\ref{subsec: Background on One-Shot NAS} for more details). The main idea is to introduce a new \emph{search space} by considering a Boolean function of the possible operations as a loss function of the NAS problem. We utilize the sparse Fourier representation of the Boolean loss function as a new \emph{search strategy} to find the (close)-optimal operations in the network. The numerical experiments show that CoNAS can discover a deep convolutional neural network with reproducible test error of $2.74 \pm 0.12\%$ for classifying of CIFAR-10 dataset. The discovered architecture outperforms the state-of-the-art methods, including DARTs~\cite{liu2018darts}, ENAS~\cite{pham2018efficient}, and random search with weight-sharing (RSWS)~\cite{li2019random} (see Table~\ref{table: CIFAR10 Comparison Table}). 
% and the baseline vanilla random search method~\cite{liu2018darts} in terms of test error, search time, model size, and number of multiply-add operations. 
% Moreover, CoNAS can achieve the comparable performance as NASNet~\cite{zoph2018learning} and AmoebaNet~\cite{real2018regularized} with less than one GPU-day of computation. 
% Our experiments on designing recurrent neural networks for language modeling are somewhat short of the state-of-the-art~\cite{zilly2017recurrent}, but we find that CoNAS still finds competitive results with less search time than previous NAS approaches. 
%Our results are exactly reproducible (having been trained with fixed pseudorandom seeds), and an implementation of CoNAS will be made publicly available post-peer review. 

\subsubsection{Theoretical Analysis} Finally, we analyze the performance of CoNAS by giving a sufficient condition on the number of performance evaluations of sub-architectures; this provides approximate bounds on training time. This, to our knowledge, is one of the first theoretical results in the NAS literature and may be of independent interest.

\subsection{Techniques}

%In this paper, we propose two algorithms  by leveraging sparse Fourier expansion of Boolean functions (referred to as Hadamard transform) for both HPO and NAS problems. 
In a nutshell, our solutions to these problems are based on~\cite{hazan2017hyperparameter} and \cite{stobbe2012learning}, which have shown how to encode set functions using a sparse polynomial basis representation. 

%\subsubsection{Hyperparameter Optimization (HPO)}
Since all HPO algorithms are computationally expensive, they should be parallelizable and scalable. Moreover, their performance should be at least as good as the random search methods~\cite{falkner2018bohb}. To achieve these goals, we use the idea of Hyperband from multi-armed bandit problems for parallelization coupled with the Harmonica algorithm by~\cite{hazan2017hyperparameter} for scalability and high performance. In particular, we first construct a Boolean cost function by binarizing the hyperparameter space, and evaluating it with a small number of (sampled) training examples. This implies the equivalency of estimating the Fourier coefficients and finding the best hyperparameters. Since only a few number of hyperparameters can result in low cost function, estimating the Fourier coefficients boils down to a sparse recovery problem from a small number of measurements. However, our results indicates that the support of non-zero coefficients show some structure. By imposing this structure, we achieve better overall test error for a given computational budget.

% from a small number of (sampled) training loss observations 
% , which is the state-of-art in multi-armed bandit problem. This makes the propode
% as the base algorithm to satisfy the first qualification. 
% To achieve the second and third criteria, we use the Harmonica trick~\cite{hazan2017hyperparameter}: we first binarize the hyperparameter space, and decompose the Fourier expansion of the (Boolean) function $f$. Finding the influential hyperparameters from a small number of (sampled) training loss observations reduces to solving a group-sparse recovery problem from compressive measurements. This leads us to better overall test error for a given computational budget.

%\subsubsection{Neural Architecture Search}
Similar to HPO, NAS is also computationally intense. To address this issue, we utilize a one-shot approach~\cite{bender2018understanding,li2019random} in which instead of evaluating several candidate architectures, a single ``base'' model is pre-trained. Then, a set of sub-networks is selected and evaluated on a validation set, and the best-performing sub-network is chosen to build a final architecture. We model the sub-network selection as a sparse recovery problem by considering a function $f$ that maps sub-architectures to a measure of performance (validation loss). Assuming that $f$ can be written as a linear combination of sparse low-degree polynomial basis functions, we can reconstruct $f$ using a small number of sub-network evaluations; hence, reducing overall computation time. Compared to~\cite{liu2018darts,pham2018efficient}, our search space allows a search over a more diverse set of candidate architectures.
% A key challenge lies in defining a suitable search space; we propose one that is considerably larger than the one used in \cite{liu2018darts} or \cite{pham2018efficient}, allowing us to (putatively) search over a more diverse set of candidate architectures.

% Following, we model the sub-network selection as a sparse recovery problem. Concretely, consider a function $f$ that maps sub-architectures to a measure of performance (validation loss). We assume that $f$ can be written as a sparse, low-degree polynomial in the (discrete) Fourier basis
%\footnote{Intuitively, this means that $f$ is well-approximated by a decision tree; for a formal proof, see~\cite{hazan2017hyperparameter}.}
% If the sparsity assumption is satisfied, then we claim the function $f$ can be reconstructed using a very small number of sub-network evaluations, thus reducing overall compute time. A key challenge lies in defining a suitable search space; we propose one that is considerably larger than the one used in \cite{liu2018darts} or \cite{pham2018efficient}, allowing us to (putatively) search over a more diverse set of candidate architectures. 

The rest of this paper is organized as follows. In Section~\ref{related}, we review some prior work on HPO and NAS problems. Section~\ref{perlim} provides some definitions and mathematical backgrounds. In Section~\ref{HPO_alg} and~\ref{HPO_exp}, we respectively introduce our HPO algorithm and the supportive experimental results. In section~\ref{NAS_alg} and~\ref{NAS_exp}, we present our NAS algorithm with rigorous theoretical analysis and experimental results, respectively. We conclude this paper in section~\ref{conc}.

\section{Related Work}
\label{related}

\subsection{Prior Work in Hyperparameter Optimization (HPO)}
HPO methods based on brute-force techniques such as exhaustive grid search are prohibitive for large hyperparameter spaces due to their exponential time complexity. One remedy for this problem is introduced by Bayesian Optimization (BO) techniques in which a prior distribution over the cost function is defined (typically a Gaussian process), and is updated after each ``observation'' (i.e., measurement of training loss) at a given set of hyperparameters~\cite{bergstra2011algorithms, hutter2011sequential, snoek2012practical, thornton2013auto, eggensperger2013towards, snoek2014input, ilievski2017efficient}. Subsequently, an acquisition function samples the posterior to form a new set of hyperparameters, and the process iterates. Despite the popularity of BO techniques, they often provide unstable performance, particularly in high-dimensional hyperparameter spaces. An alternative approach to BO is Random Search (RS), with efficient computational time, strong ``anytime'' performance with easy parallel implementation~\cite{bergstra2012random}.

% Traditionally, ML practitioners have solved the HPO problem via brute-force techniques such as grid search over $X$. This strategy quickly runs into exponentially increasing computation costs in the number of hyperparameter space. 
% As a solution, Bayesian Optimization (BO) techniques have been proposed. These assume a certain prior distribution over the cost function $f(x)$ and updates the posterior distribution with each new ``observation'' (or measurement of training loss) at a given set of hyperparameters~\cite{bergstra2011algorithms, hutter2011sequential, snoek2012practical, thornton2013auto, eggensperger2013towards, snoek2014input, ilievski2017efficient}. Subsequently, an acquisition function samples the posterior to form a new set of hyperparameters, and the process iterates. Despite the popularity of BO techniques, they often provide unstable performance, particularly in high-dimensional hyperparameter space. 
% An alternative approach to BO is Random Search (RS), with efficient computational time, strong ``anytime'' performance with easy parallel implementation~\cite{bergstra2012random}. 
% which not only provides computational efficiency compared to grid search, but also strong ``anytime'' performance with easy parallel implementation~\cite{bergstra2012random}. 

Multi-armed bandit (MAB) methods adapt the random search strategy to allocate the different resources for randomly chosen candidates to speed up the convergence. However, random search and the BO approaches spend full resources. Successive Halving (SH) and Hyperband are two examples of MAB methods, which ignore the hyperparameters with poor performance in the early state~\cite{jamieson2016non,li2017hyperband,kumar2018parallel}. In contrast with BO techniques (which are hard to parallelize), the integration of BO and Hyperband achieve both advantages of guided selection and parallelization~\cite{wang2018combination, falkner2018bohb, bertrandhyperparameter}.
% utilize the multi-armed bandit approach to random search, selecting more candidates than random search with the same amount of budget by pruning poorly-performing hyperparameters in the early state~\cite{jamieson2016non,li2017hyperband,kumar2018parallel}. %Decision-theoretic methods based on random search quickly finds promising hyperparameter settings, but struggle with achieving good final performance compare to the model-based approach. 
% In contrast with BO techniques (which are hard to parallelize), the integration of BO and Hyperband achieve both advantages of guided selection and parallelization~\cite{wang2018combination, falkner2018bohb, bertrandhyperparameter}.

Gradient descent methods~\cite{bengio2000gradient,maclaurin2015gradient,luketina2015scalable,fu2016drmad, franceschi2017forward} (or more broadly, meta-learning approaches) have also been applied to solve the HPO problem, but these are only suitable to optimize continuous hyperparameters. Since this is a very vast area of current research, we do not compare our approach with these techniques.

While BO dominates model-based approaches, a recent technique called \emph{Harmonica} utilize a \emph{spectral} approach by applying sparse recovery techniques on a Boolean version of the objective function. This gives Harmonica the benefit of reducing the dimension of the hyperparameter space by quickly finding influential hyperparameters~\cite{hazan2017hyperparameter}.

% introduced a \emph{spectral} approach, applying ideas from \emph{sparse recovery} on a Boolean version of the objective function. 
% Using this approach, Harmonica provides the unique benefit of reducing the dimensionality of hyperparameter space by quickly finding highly influential hyperparameters, following which other standard (search or optimization) techniques can be used~\cite{hazan2017hyperparameter}.

\subsection{Prior Work in Neural Architecture Search}

Early NAS algorithms were using reinforced learning (RL) based controllers~\cite{zoph2018learning}, evolutionary algorithms \cite{real2018regularized}, or sequential model-based optimization (SMBO)~\cite{liu2018progressive}. The performance of these methods is competitive with the manually-designed architectures such as deep ResNets~\cite{he2016deep} and DenseNets~\cite{huang2017densely}. However, they require substantial computational resources, e.g., thousands of GPU-days. Other NAS approaches have focused on boosting search speeds by proposing novel search strategies, such as differentiable search technique~\cite{cai2018proxylessnas,liu2018darts,noy2019asap,luo2018neural,xie2018snas} and random search via sampling sub-networks from a one-shot super-network~\cite{bender2018understanding, li2019random}. In particular, DARTs~\cite{liu2018darts} %, 
% which our algorithm, CoNAS 
is based on a bilevel optimization by relaxing the discrete architecture search space to a differentiable one via softmax operations. This relaxation makes it faster by orders of magnitude while achieving competitive performance compared to previous works~\cite{zoph2016neural, zoph2018learning, real2018regularized, liu2018progressive}. 

Other recent methods include RL approaches via weight-sharing ~\cite{pham2018efficient}, network transformations~\cite{cai2018efficient,elsken2018efficient,jaderberg2017population,jin2018efficient,liu2017hierarchical, hu2019efficient}, and random exploration~\cite{li2018massively,li2019random,sciuto2019evaluating,xie2019exploring}. None of these methods has explored utilizing the sparse recovery techniques for NAS. The closest approach to ours but a different objective is the one proposed by ~\cite{stobbe2012learning} which learns a sparse graph from a small number of random cuts.
% The closest approach to ours is the one proposed by~\cite{stobbe2012learning},
% which learns a sparse graph from a small number of random cuts.
% which is based on learning sparse sub-networks from a small number of random cuts.
While \cite{stobbe2012learning} emphasizes on \emph{linear} measurements, CoNAS takes a different perspective by focusing on measurements that map sub-networks to performance, which are fundamentally \emph{nonlinear}.
% However, \cite{stobbe2012learning} emphasizes on \emph{linear} measurements, but CoNAS takes a different perspective by focusing on measurements that map sub-networks to performance, which are fundamentally \emph{nonlinear}. Moreover, our theoretical bounds is tighter, and leads to improved results in terms of sample complexity. 

In \cite{bender2018understanding}, the authors  provided an extensive experimental analysis on one-shot architecture search based on weight-sharing and correlation between the one-shot model (super-graph) and stand-alone model (sub-graph). The authors of~\cite{li2019random} proposed a simplified training procedures compare to \cite{bender2018understanding} without neither super-graph stabilizing techniques such as path dropout schedule on a direct acyclic graph (DAG) nor ghost batch normalization. We defer the details of super-graph training from~\cite{li2019random} in Section~\ref{subsec: Background on One-Shot NAS}.

% Compared to work in~\cite{bender2018understanding}, authors in~\cite{li2019random} have proposed a simplified training procedures by path dropout schedule on a direct acyclic graph (DAG) and ghost batch normalization~\cite{bender2018understanding}. 
% without stabilizing techniques (e.g., path dropout schedule on a direct acyclic graph (DAG) and ghost batch normalization) from \cite{bender2018understanding}. 
% As the final performance of the discovered architecture heavily relies on hyperparameter settings, \cite{li2019random} exactly accords hyperparameters and data augmentation techniques to DARTs for their experiments. 
The combination of random search via one-shot models with weight-sharing provides the best competitive baseline results reported in the one-shot NAS literature. Our CoNAS approach improves upon these reported results.
% A lengthier description is available in Appendix~\ref{appendix:background one-shot}.

% Finally, \cite{stobbe2012learning} propose learning sparse sub-networks from a small number of random cuts; they also leverage ideas from compressive sensing and provide theoretical upper bounds for successful recover. Our CoNAS approach is directly inspired from their seminal work. However, we emphasize essential differences: while \cite{stobbe2012learning} emphasize \emph{linear} measurements, CoNAS takes a different perspective by focusing on measurements that map sub-networks to performance, which are fundamentally \emph{nonlinear}. Moreover, our theoretical bounds use better Fourier-RIP bounds, and lead to improved results in terms of measurement complexity. 

\section{Preliminaries}
\label{perlim}
% In this section, we provide some notations and  mathematical definitions.

\subsection{Fourier Analysis of Boolean Functions}
% We introduce brief description on Fourier analysis of Boolean functions that will be used throughout the paper. We follow the treatment given in~\cite{o2014analysis}. 
Throughout this paper, we denote the vectors with bold letters. Also, $[n]$ denotes the set $\{1,2\dots,n\}$. A {real-valued Boolean function} is one that maps $n$-bit binary vectors (i.e., vertices of a hypercube) to real values: $f : \{-1,1\}^n \rightarrow \mathbb{R}$. Such functions can be represented in a basis comprising real multilinear polynomials called the \emph{Fourier} basis as follows~\cite{o2014analysis}. 
% (We denote the vectors with bold letters. Also, $[n]$ denotes the set $\{1,2\dots,n\}$.)
%The Fourier basis (or monomial) corresponding to any subsets $S \subseteq [n]$ (a set of indices) is defined as following: 
\begin{definition}
\label{def:fourier basis}
For $S \subseteq [n]$, define the parity function $\chi_S: \{-1, 1\}^n \rightarrow \{-1, 1\}$ such that $\chi_S(\bm{\alpha}) = \prod_{i \in S} \alpha_i$. Then, the Fourier basis is the set of all $2^n$ parity functions $\{\chi_S\}$.
\end{definition}
The key fact is that the basis of parity functions forms an $K$-bounded orthonormal system (BOS) with $K=1$. That is: 
\begin{align}
\label{eq:BOS}
\langle \chi_S, \chi_T \rangle = 
\begin{cases}
    1, & \text{if } S = T\\
    0, & \text{if } S \neq T
\end{cases}
\:\:\:\:\:\:\:\:\:\:\: \textrm{and} \\ \sup_{\bm{\alpha} \in \{-1, 1\}^n} |\chi_S(\bm{\alpha})| \leq 1 \:\:\: \textrm{for all }S \subseteq [n],
\end{align}
As it has been shown in~\cite{o2014analysis}, any Boolean function $f$ has a unique Fourier representation as $f(\bm{\alpha}) = \sum_{S \subseteq [n]} \hat{f}(S) \chi_S(\bm{\alpha})$, with Fourier coefficients $\hat{f}(S) = \mathbb{E}_{\bm{\alpha}\in\{-1,1\}^n}[f(\bm{\alpha})\chi_S(\bm{\alpha})]$ where
% \begin{align}
% \label{eq:Fourier Series}
% f(\bm{\alpha}) = \sum_{S \subseteq [n]} \hat{f}(S) \chi_S(\bm{\alpha})
% \end{align}
% where
% \begin{align}
% \label{eq:Fourier Coefficients}
% \hat{f}(S) = \mathbb{E}_{\bm{\alpha}\in\{-1,1\}^n}[f(\bm{\alpha})\chi_S(\bm{\alpha})]
% \end{align}
expectation is taken with respect to the uniform distribution over the vertices of the hypercube. 
% and
% \begin{align}
% \label{eq:BOS2}
%     \sup_{\bm{\alpha} \in \{-1, 1\}^n} |\chi_S(\bm{\alpha})| \leq 1 \:\:\:\: \textrm{for all }S \in [n] 
% \end{align}
For many objective function in machine learning, the Fourier spectrum of the function is concentrated on monomials of small degree ($\leq d$) (e.g., decision trees~\cite{hazan2017hyperparameter}). Leveraging this property simplifies the Fourier expansion by limiting the number of basis functions. Let $\mathcal{P}_d \subseteq 2^{[n]}$ be a fixed collection of Fourier basis %with polynomial size $d$ or less 
such that $\mathcal{P}_d \coloneqq \{\chi_S \subseteq 2^{[n]}: |S| \leq d\}$. Then, $\mathcal{P}_d \subseteq 2^{[n]}$ induces a function space, consisting of all functions of order $d$ or less, denoted by $\mathcal{H}_{\mathcal{P}_d} := \{f : Supp[\hat{f}] \subseteq \mathcal{P}_d\}$. For example, $\mathcal{P}_2$ allows us to express the function $f$ with at most $\sum_{l=0}^{d} \binom{n}{l} \equiv \mathcal{O}(n^2)$ Fourier coefficients. Next, we define the restriction of a function $f$ to some index set $J$.

%Lastly, Boolean function allows restricting the hypercube domain to subcube with the operation called \emph{restriction}.
% Lastly, if we have prior knowledge of some set of bits ${J}$, we use an operation called

\begin{definition}%[Restriction~\cite{o2014analysis}]
\label{def:restriction}
Let $f: \{-1, 1\}^n \rightarrow \mathbb{R}$, $(J, \overline{J})$ be a partition of $[n]$, and $\mathbf{z} \in \{-1, 1\}^{\overline{J}}$. The restriction of $f$ to $J$ using $z$ denoted by $f_{J|\mathbf{z}}: \{-1, 1\}^J \rightarrow \mathbb{R}$ is the subfunction of $f$ given by fixed coordinates in $\overline{J}$ to the values of $\mathbf{z}$.
\end{definition}

\section{Hyperparameter Optimization}
\label{HPO_alg}

In this section, we present our HPO algorithm. We restrict our attention to discrete domains (we assume that continuous hyperparameters have been appropriately binned). Let $f: \{-1, 1\}^n \mapsto \mathbb{R}$ be the loss function we want to optimize. Assume there exists $k$ different types of hyperparameters. In other words, we allocate $n_i$ bits to the $i^{\textrm{th}}$ hyperparameter category such that $\sum_{i=1}^{k} n_i=n$. The task of HPO involves searching the global minimizer(s) of the following optimization problem as the best hyperparameters: 
\begin{align}
    \mathbf{\alpha^*} = \argmin_{\mathbf{\alpha}\in\{-1,1\}^n}f(\mathbf{\alpha}) .
\label{opt_problem}
\end{align}

% \subsection{PGSR-HB (Polynomial Group-Sparse Recovery within HyperBand)}

% {\color{red}To solve (\ref{opt_problem})}, 
We propose \emph{Polynomial Group-Sparse Recovery within Hyperband} (PGSR-HB), a new HPO search algorithm which significantly reduces the size of hyperparameter space. We combine Hyperband, the multi-armed bandit method that balances exploration and exploitation from uniformly random sampled hyperparameter configurations, with a group sparse version of Polynomial Sparse Recovery. 
% , which is the main component of the spectral decomposition-based Harmonica method of HPO.
Algorithm~\ref{alg:PGSRHB} shows the pseudocode of PGSR-HB.

\begin{algorithm}[t]
\caption{\textsc{PGSR-HB}}
\label{alg:PGSRHB}
\begin{algorithmic}[1]
\State\textbf{Inputs:} Resource $R$, scaling factor $\eta$, total cycle $c$ 
\State\textbf{Initialization:} $s_{max}=\lfloor \log_{\eta}(R)\rfloor$, $B=(s_{max}+1)R$, input history $H_{input}=\emptyset$, output history $H_{output}=\emptyset$
\For{$round=1:c$}
\For{$s \in \{s_{max}, s_{max}-1,\ldots,0\}$}
\State $n=\lceil \frac{B}{R} \frac{\eta^s}{(s+1)}\rceil$, $r=R\eta^{-s}$
\State T = {\fontfamily{pcr}\selectfont PGSR\_Sampling(n)}
\For{$i \in \{0,\ldots,s\}$}
\State $n_i = \lfloor n\eta^{-i}\rfloor$
\State $r_i = r\eta^i$
\State $L = \{f(t,r_i): t \in T\}$
\State $H_{input,r_i} \leftarrow H_{input,r_i} \bigcup T$
\State $H_{output,r_i} \leftarrow H_{output,r_i} \bigcup L$
\State $T = \text{top}_k(T,L,\lfloor \frac{n_i}{\eta} \rfloor)$
\EndFor
\EndFor
\EndFor
\State\textbf{return} Configuration with the smallest loss
\Algphase{Sub-algorithm - PGSR Sampling}
\State\textbf{Input:} $H_{input}$, $H_{output}$, sparsity $s$, polynomial degree $d$, minimum observations $T$, randomness ratio $\rho$
\If{every $|H_{output,r}| < T$} \Return random sample from original domain of $f$.
\EndIf
\State Pick $H_{input,r}$ and $H_{output,r}$ with largest $r$: $|H_{output,r}|\geq T$.
\State Group Fourier basis based on hyperparameter structure.
\State Solve $$\mathbf{\alpha}^* = \argmin_{\mathbf{\alpha}} \frac{1}{2}\|\mathbf{y} - \sum_{l=1}^{m}\Psi^{l} \mathbf{\alpha}^{l}\|_{2}^{2} + \lambda \sum_{l=1}^{m} \sqrt{p_{l}} \|\mathbf{\alpha}^{l}\|_2$$
\State Let $S_1, \ldots S_s$ be the indices of the largest coefficient of $\alpha$. Then, $g(\mathbf{\alpha}) = \sum_{i \in [s]} \alpha_{S_i} \chi_{S_i}(\mathbf{\alpha})$ and $J = \bigcup_{i=1}^{s}S_i$
\State With probability $\rho$, return random sample from original domain of $f$; else return random sample from reduced domain of $f_{J,\mathbf{\alpha}^{*}}$.
\end{algorithmic}
\end{algorithm}

PGSR-HB adopts the decision-theoretic approach of Hyperband, but with the additional tracking the history of all loss values from different resources. Lines 7-14 in Algorithm~\ref{alg:PGSRHB}, illustrates the Successive Halving (SH) subroutine of Hyperband, a early performance (e.g. validation loss) in the process of training indicates which hyperparameter configurations are worth investing further resources, and which ones are fit to discard. We defer the pseudocode of SH and Hyperband in the Appendix~\ref{appendix: SH and Hyperband}. 
% {\color{red}following the assumption that the performance of different hyperparameter choices in the process of training indicates which configurations are worth investing further resources, and which ones are fit to discard. We defer the pseudocode of SH in the appendix.} 

Now, let $R$ denote the (units of computational) resource to be invested in one round to observe the final performance of the model. Let $\eta$ denote a scaling factor and $c$ be the total number of rounds. Defining $s_{max}=\log_{\eta}R$, the total budget spent from SH is given by $B=(s_{max}+1)R$. In addition, Line 6 of Algorithm~\ref{alg:PGSRHB} invokes a sub-routine (discussed below) to sample $n$ configurations given as follows:
\begin{align}
    n = \lceil \frac{B}{R} \frac{\eta^s}{(s+1)}\rceil
\end{align}
Also, the test loss is computed with $r = R\eta^{-s}$
% \begin{align}
    % r = R\eta^{-s}
% \end{align}
epochs of training. The function $f(t,r_i)$ in Algorithm~\ref{alg:PGSRHB} (Line 10) returns the intermediate test loss of a hyperparameter configuration $t$ with $r_i$ training epochs. Since the test loss is a metric to measure the performance of the model, the algorithm keeps only the top $\frac{1}{\eta}$ configurations (Line 13) and repeats the process by increasing number of epochs by the factor of $\eta$ until $r$ reaches to resource $R$. While SH introduces a new hyperparameter $s$, SH aggressively explores the hyperparameter space as $s$ close to $s_{max}$. SH with $s=0$ is equivalent to random search (aggressive exploitation).\footnote{The algorithm with one cycle contains $(s_{max}+1)$ subroutines of SH during which it tries different levels of exploration and exploitation with all possible values of $s$.}

\subsection{PGSR Sampling}

As PGSR-HB collects the outputs of the function $f$, PGSR-Sampling estimates Fourier coefficients of the function $f$ using techniques from sparse recovery to reduce the hyperparameter space. 
% While Harmonica does not explicitly address how to discretize continuous hyperparameters, 
We now introduce a simple mathematical expression that efficiently induces additional sparsity in the Fourier representation of $f$. Given $k$ categories ($i=1,\ldots,k$), let $g$ and $h$ be two functions that map binary numbers $\mathbf{x}_i$ and $\mathbf{y}_i$ with respectively $\alpha_i$ and $\beta_i$ digits to the set of integers with cardinality $2^{\alpha_i}$ and $2^{\beta_i}$.
% Similarly, let $h$ be another function that maps a binary number $\mathbf{y}_i$ with $\beta_i$-digits to the set of numbers with cardinality $2^{\beta_i}$ which are evenly spaced in (0,1]. 
Then we express the $i^{th}$ real-valued hyperparameter, $hp_i$ with corresponding binary digits $\mathbf{x}_i$ and $\mathbf{y}_i$ in a log-linear manner as follows:
\begin{align}
\label{eq:HpRepresentation}
    hp_i = 10^{g(\mathbf{x}_i)} \cdot h(\mathbf{y}_i)
\end{align}

Our experimental results show how the above nonlinear binning representation induces sparsity on function $g$.
% which captures the value's order of magnitude. As PGSR returns the features regard to the function $g$, the new representation efficiently reduces the hyperparameter space. 
While PSR in Harmonica recovers the Boolean function with Lasso~\cite{tibshirani1996regression}, the intuitive extension (arising from the above log-linear representation) is to replace sparse recovery with Group Lasso~\cite{yuan2006model}. This is used in Line 23 of Algorithm~\ref{alg:PGSRHB} as we group them based on the $g$ and $h$ based on the  hyperparameter categories. Let $\mathbf{y}\in \mathbb{R}^m$ be the observation vector and $\mathbf{x} \in \{-1, 1\}^n$, and devide the hyperparameters into $k+k$ groups (corresponding to functions $g$ and $h$). Let $\Psi^{l}$ is the submatrix of $\Psi \in \mathbb{R}^{m \times (\sum_{i=1}^{d}\binom{n}{i})}$ where its columns match the $l^{th}$ group. Similarly, $\alpha^{l}$ is a weight vector corresponding to the submatrix $\Psi^{l}$ and $p_{l}$ denotes the length of vector $\alpha^{l}$. In order to construct the submatrices as the Fourier basis with the hyperparameter structure, we assume that there exist a set of groups $G = \{g_1, \ldots, g_k, h_1, \ldots, h_k\}$ as defined above. If there are $\gamma$ possible combinations of groups from $G$ such that a $d$-degree Fourier basis exists ($\gamma = \sum_{i=1}^{d}\binom{2k}{i}$), we derive the $k$ submatrices $\Psi^{1}, \ldots, \Psi^{\gamma}$ using Definition \ref{def:fourier basis}. Then the problem is equivalent to a convex optimization problem known as the Group Lasso:
%solving the Group Lasso with additional information regard to the structure of hyperparameters with the following equation:
\begin{align}
\label{eq:GroupLasso}
    \min_{\alpha} \frac{1}{2}\|y - \sum_{l=1}^{m}\Psi^{l} \alpha^{l}\|_{2}^{2} + \lambda \sum_{l=1}^{m} \sqrt{p_{l}} \|\alpha^{l}\|_2
\end{align}

Finally, the algorithm requires the input $\rho$, representing a reset probability parameter that produces random samples from the original reduced hyperparameter space. This parameter prevents the biased in different PGSR stages.
% since the measurements with substantial resources mostly arise from the later stages of Successive Halving. 

\subsection{Differences between PGSR-HB and Harmonica}
The standard Harmonica method samples the measurements under a uniform distribution. It then runs the search algorithm to recover the function $f$ with PSR (the sparse recovery through $l_1$ penalty, or standard Lasso). Moreover, the number of randomly sampled measurements and its resources (training epochs) needs to be given before starting the search algorithm. we note that the reliability of measurements hugely depends on the number of resources used on each sampled point. While investing enormous resources in recovering Fourier coefficients guarantees the success of the Lasso recovery, this is inefficient with respect to total budget. On the other hand, collecting the measurements with small resources make PSR fail to provide the correct guidance for the outer search algorithm. Since PGSR-HB gathers all the function outputs -- from cheap resources to the most expensive resources -- PGSR-HB eliminates the need to set an explicit number of samples and training epochs as in Harmonica. We have also tried other penalties, and observed that the regularized regression tends to learn slower than the models without a regularization; thus, misleading the search algorithm with the worst performance.
% We have experimented with other penalties than the standard L1-penalty: for example, Tikhonov regularization prevents model overfitting particularly in deep architectures. However, the regularized regression tends to learn slower than the model without a regularization, consequently misleading the search algorithm with the worst performance. 
% Since PGSR-HB gathers all the function outputs -- from cheap resources to the most expensive resources -- PGSR-HB eliminates the need to set an explicit number of samples and training epochs as in Harmonica. 
While the experimental results in \cite{hazan2017hyperparameter} shows promising results in finding the influential categorical hyperparameters such as presence/absence of the Batch-normalization layer, it cannot be used directly in optimizing the numerical hyperparameters such as learning rate, weight decay, and batch size. One the other hand, PGSR-HB overcomes this limitation of Harmonica with the log-linear representation in~\eqref{eq:HpRepresentation} and Group Lasso~\eqref{eq:GroupLasso}.
%or determining the descent algorithm (stochastic gradient descent vs. Adam) [both of which can be represented using binary variables]. However, 

\section{HPO Experimental Results}
\label{HPO_exp}
\subsection{Robustness Test}

We verify the robustness of PGSR-HB by generating a test loss surface picking two hyperparameter categories as shown in Figure~\ref{fig:lossSurface}. We calculate the test loss by training 120 epochs for classification of CIFAR-10 data set. We use the convolutional neural network architecture from the cuda-convnet-$82\%$ model that has been used in previous work (\cite{jamieson2016non} and \cite{li2017hyperband}). The range of learning rate and the weight-decay on the first convolutional layer is set to be in $10^{-6}$ to $10^{2}$. We use the log scale on both horizontal and vertical axis for the visualize of the loss surface. 
% with more natural interpretation and dynamic range on the test loss. 

Table~\ref{table:LrvsConv1} compares the performance of PGSR and PSR with ~\eqref{eq:HpRepresentation}, and PSR with evenly spaced hyperparameter values in log scale. The third and fourth columns in Table~\ref{table:LrvsConv1} list the reduced hyperparameter space for learning rate and $l_2$ regularization coefficient for first convolution layer for each algorithm.
From the comparison between the reduced search space and hyperparameter loss surfaces from Figure~\ref{fig:lossSurfaceXaxis} and Figure~\ref{fig:lossSurfaceYaxis}, we observe \eqref{eq:HpRepresentation} reduce the search domain more accurately than without \eqref{eq:HpRepresentation}. 
% {\color{red}This experiment shows that~\eqref{eq:HpRepresentation} can reduce the search space more than the conventional method.} 
Imposing structured sparsity on the hyperparameters by grouping them not only helps PGSR to return the correct guidance, but also it stabilize the lasso coefficient $\lambda$ as shown in the test loss surfaces (Figure~\ref{fig:lossSurfaceXaxis} and Figure~\ref{fig:lossSurfaceYaxis}) with PGSR results in Table~\ref{table:LrvsConv1}.

\begin{figure}[t]
\begin{center}
    \setlength{\tabcolsep}{.1pt}
    \renewcommand{\arraystretch}{.1}
    \begin{tabular}{cc}
    \includegraphics[trim = 30mm 105mm 77mm 105mm, clip, width=0.8\linewidth]{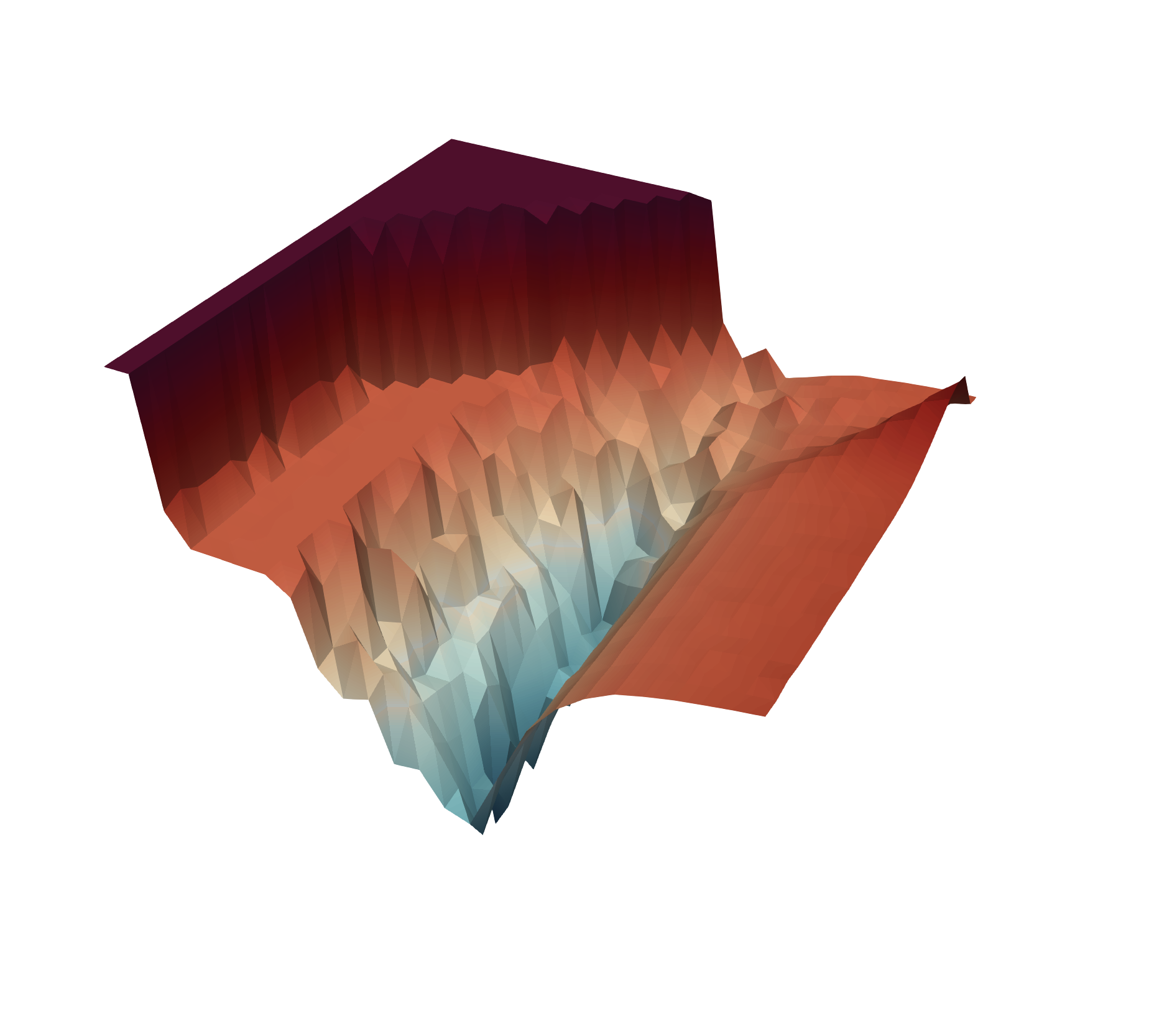}
    \end{tabular}
\end{center}
\caption{\emph{Test loss surface with two hyperparameters. Learning rate vs conv1 l2 penalty.}}
\label{fig:lossSurface}
\end{figure}

\begin{figure}[t]
\begin{center}
    \setlength{\tabcolsep}{.1pt}
    \renewcommand{\arraystretch}{.1}
    \begin{tabular}{cc}
    \includegraphics[trim = 35mm 40mm 35mm 40mm, clip, width=.8\linewidth]{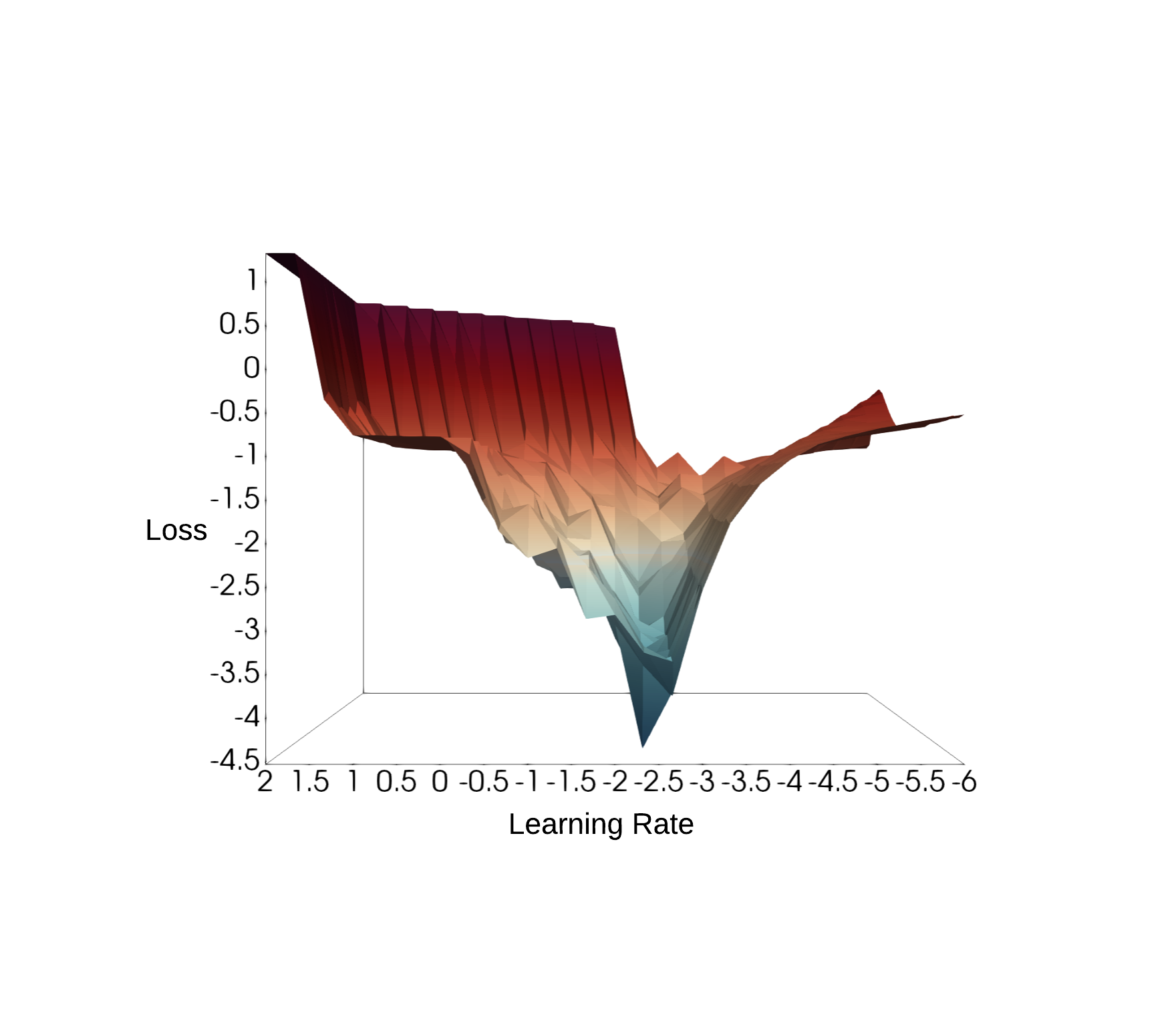}
    \end{tabular}
\end{center}
\caption{\emph{The view from learning rate axis.}}
\label{fig:lossSurfaceXaxis}
\end{figure}

\begin{figure}[t]
\begin{center}
    \setlength{\tabcolsep}{.1pt}
    \renewcommand{\arraystretch}{.1}
    \begin{tabular}{cc}
    \includegraphics[trim = 35mm 40mm 35mm 40mm, clip, width=0.8\linewidth]{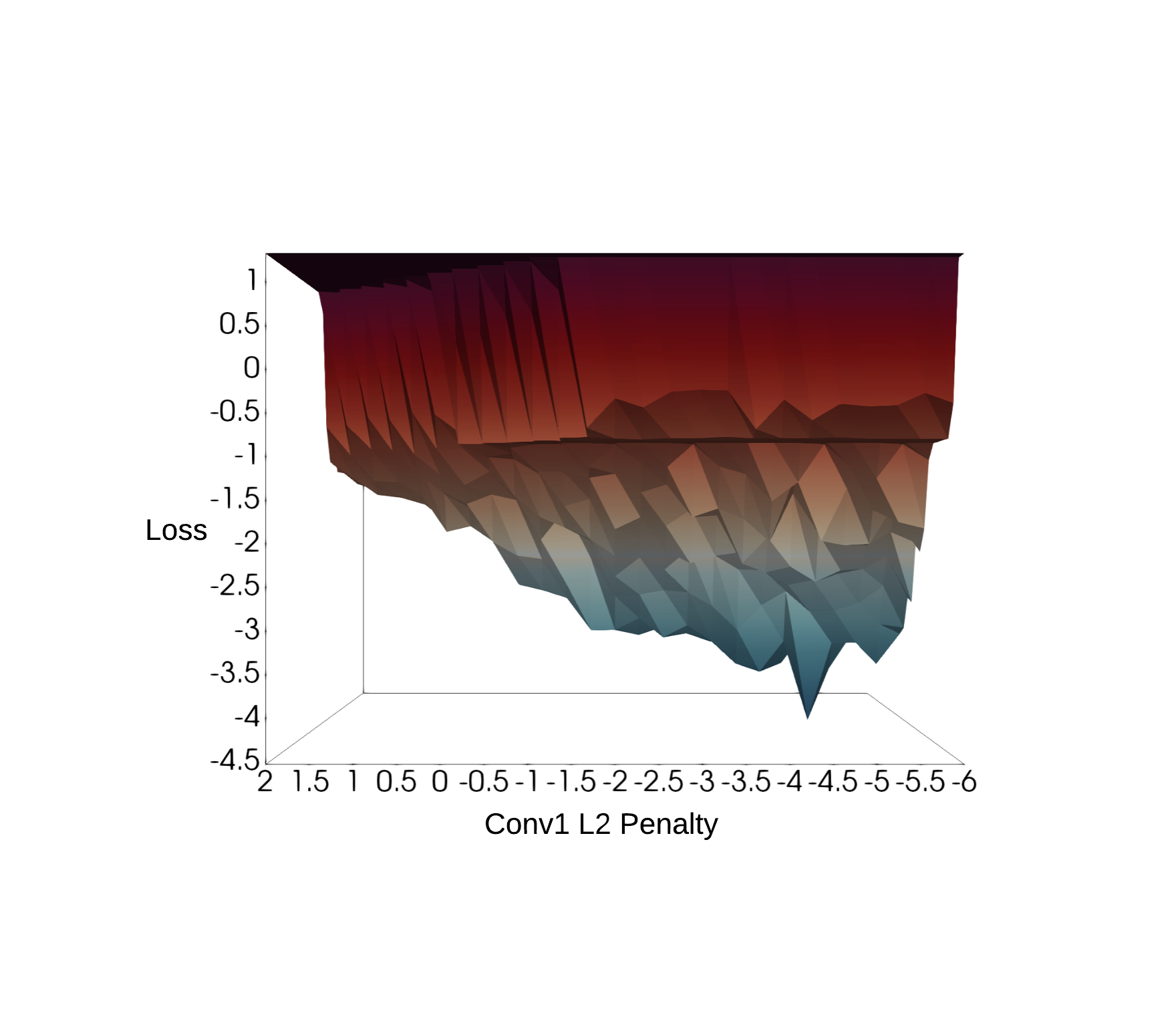}
    \end{tabular}
\end{center}
\caption{\emph{The view from conv1 l2 penalty axis.}}
%\vspace{0.5cm}
\label{fig:lossSurfaceYaxis}
\end{figure}

%\vspace{1.0cm}
% \begin{table}[!t]
% \centering
% \caption{Guidance Comparison on Learning Rate and Conv1 L2}
% \vspace{1em}
% \label{table:LrvsConv1}
% \begin{center}
% \begin{tabular}{|c|c|c|c|c|c|}
% \hline
% Method &$\lambda$& Learn Rate & Conv1 Penalty\\
% \hline
% PGSR & $0.5$ & $[\mathbf{10^{-3}, 10^{-2}}]$ & $[\mathbf{10^{-5}, 10^{-4}}]$ \\
% \hline
% PGSR & $1.0$ & $[\mathbf{10^{-3}, 10^{-2}}]$ & $[\mathbf{10^{-5}, 10^{-4}}]$ \\
% \hline
% PGSR & $2.0$ & $[\mathbf{10^{-3}, 10^{-2}}]$ & $[\mathbf{10^{-5}, 10^{-4}}]$ \\
% \hline
% PSR & $0.5$ & $[\mathbf{10^{-3}, 10^{-2}}]$ & $[10^{-6}, 10^{2}]$ \\
% \hline
% PSR & $1.0$ & $[10^{-4}, 10^{-3}]$ & $[10^{-3}, 10^{-2}]$ \\
% \hline 
% PSR & $2.0$ & $[10^{0}, 10^{2}]$ & $[10^{-3}, 10^{-2}]$ \\
% \hline
% PSR w/o~\eqref{eq:HpRepresentation} & $0.5$ & $[10^{-4}, 10^{-2}]$ & $[10^{-6}, 10^{-3}]$ \\ 
% \hline 
% PSR w/o~\eqref{eq:HpRepresentation} & $1.0$ & $[10^{-4}, 10^{-2}]$ & $[10^{-6}, 10^{-3}]$ \\
% \hline
% PSR w/o~\eqref{eq:HpRepresentation} & $2.0$ & $[10^{-4}, 10^{-2}]$ & $[10^{-6}, 10^{-4}]$ \\ 
% \hline
% \end{tabular}
% \end{center}
% \end{table}

\begin{table}[!h]
    \centering
    \captionsetup{width=\linewidth}
    \caption{Reduced Search Domain Comparison on Learning Rate and Conv1 L2} 
    %\vspace{1em}
    \label{table:LrvsConv1}
    \begin{threeparttable}
    % \resizebox{\linewidth}{!}{
    \begin{tabular}{c c c c c c}
        \toprule
        \multicolumn{1}{l}{\textbf{Method}} & $\lambda$ & Learning Rate & Conv1 Penalty \\ 
        \midrule

        \multicolumn{1}{l}{PGSR} & $0.5$ & $[\mathbf{10^{-3}, 10^{-2}}]$ & $[\mathbf{10^{-5}, 10^{-4}}]$ \\

        \multicolumn{1}{l}{PGSR} & $1.0$ & $[\mathbf{10^{-3}, 10^{-2}}]$ & $[\mathbf{10^{-5}, 10^{-4}}]$ \\

        \multicolumn{1}{l}{PGSR} & $2.0$ & $[\mathbf{10^{-3}, 10^{-2}}]$ & $[\mathbf{10^{-5}, 10^{-4}}]$ \\

        \multicolumn{1}{l}{PSR} & $0.5$ & $[\mathbf{10^{-3}, 10^{-2}}]$ & $[10^{-6}, 10^{2}]$ \\

        \multicolumn{1}{l}{PSR} & $1.0$ & $[10^{-4}, 10^{-3}]$ & $[10^{-3}, 10^{-2}]$ \\

        \multicolumn{1}{l}{PSR} & $2.0$ & $[10^{0}, 10^{2}]$ & $[10^{-3}, 10^{-2}]$ \\
        \multicolumn{1}{l}{PSR w/o~\eqref{eq:HpRepresentation}} & $0.5$ & $[10^{-4}, 10^{-2}]$ & $[10^{-6}, 10^{-3}]$ \\ 
        \multicolumn{1}{l}{PSR w/o~\eqref{eq:HpRepresentation}}  & $1.0$ & $[10^{-4}, 10^{-2}]$ & $[10^{-6}, 10^{-3}]$ \\
        \multicolumn{1}{l}{PSR w/o~\eqref{eq:HpRepresentation}}  & $2.0$ & $[10^{-4}, 10^{-2}]$ & $[10^{-6}, 10^{-4}]$ \\ 
        
        % % \multicolumn{1}{l}{Shake-Shake~\cite{devries2017improved}} & $2.56\pm0.07$ & 26.2 & - \\ 
        % \multicolumn{1}{l}{\cite{xie2018snas}} & $16.5$ & $1.98$ & $3.73$ & $2.8$ & $1.5$ \\
        % \multicolumn{1}{l}{\cite{liu2018progressive}} & $15.9$ & $1.83$ & $3.72$ & $3.2$ & $150$ \\
        % \multicolumn{1}{l}{\cite{zoph2018learning}} & $15.8$ & $1.96$ & $3.71$ & $3.3$ & $1800$\\
        % \multicolumn{1}{l}{\cite{liu2018darts}} & $15.8$ & $1.85$ & $3.68$ & $3.4$ & $1$ \\
        % \multicolumn{1}{l}{\cite{real2018regularized}} & $15.9$ & $1.93$ & $3.8$ & $3.2$ & $3150$ \\
        % \multicolumn{1}{l}{\cite{noy2019asap}} & $15.6$ & $1.81$ & $3.73$ & $2.5$ & 0.2 \\ 
        % \midrule
        % \multicolumn{1}{l}{CoNAS} & $15.9$ & 1.44 & $4.11$ & 2.3 & 0.4 \\
        \bottomrule
    \end{tabular}
    % }
    \end{threeparttable}
\end{table}

% \subsection{Performance Comparison on CIFAR-10}
% \begin{table}[h]
% \caption{CNN Test Loss and Accuracy on CIFAR-10}
% \vspace{1em}
% \label{table:CNNPerformance}
% \begin{center}
% \begin{tabular}{|c|c|c|c|c|}
% \hline
% Algorithm & RS 2x & SH & HB & PGSR-HB \\
% \hline
% Loss (I) & $0.7118$ & $0.7001$ & $0.7150$ & $\mathbf{0.6455}$ \\
% Acc (I) & $81.17\%$ & $79.69\%$ & $78.74\%$ & $\mathbf{82.79\%}$ \\
% \hline
% Loss (II) & $0.6988$ & $0.7179$ & $0.6921$ & $\mathbf{0.6764}$ \\
% Acc (II) & $79.51\%$ & $79.30\%$ & $81.67\%$ & $\mathbf{83.00\%}$ \\
% \hline
% Loss (III) &$0.6850$& $0.6747$& $0.6960$& $\mathbf{0.6467}$\\
% Acc (III) &$79.02\%$& $79.80\%$& $\mathbf{81.47\%}$& $80.39\%$\\
% \hline
% Loss (IV) &$0.7293$&$\mathbf{0.6499}$&$0.7215$&$0.6619$ \\
% Acc (IV) &$77.70\%$&$80.68\%$&$80.81\%$&$\mathbf{81.64\%}$\\
% \hline
% \end{tabular}
% \end{center}
% \end{table}

\begin{table}[!h]
    \centering
    \captionsetup{width=\linewidth}
    %\vspace{1mm}
    \caption{CNN Test Loss and Accuracy on CIFAR-10} %\vspace{1em}
    \label{table:CNNPerformance}
    \begin{threeparttable}
    % \resizebox{\linewidth}{!}{
    \begin{tabular}{c c c c c}
        \toprule
        \multicolumn{1}{l}{\textbf{Algorithm}} & RS 2x & SH & HB & PGSR-HB \\ 
        \midrule

        \multicolumn{1}{l}{Loss (I)} & $0.7118$ & $0.7001$ & $0.7150$ & $\mathbf{0.6455}$ \\
        \multicolumn{1}{l}{Acc (I)} & $81.17\%$ & $79.69\%$ & $78.74\%$ & $\mathbf{82.79\%}$ \\
        
        \midrule
        \multicolumn{1}{l}{Loss (II)} & $0.6988$ & $0.7179$ & $0.6921$ & $\mathbf{0.6764}$ \\
        \multicolumn{1}{l}{Acc (II)} & $79.51\%$ & $79.30\%$ & $81.67\%$ & $\mathbf{83.00\%}$ \\
        \midrule
        \multicolumn{1}{l}{Loss (III)} &$0.6850$& $0.6747$& $0.6960$& $\mathbf{0.6467}$\\
        \multicolumn{1}{l}{Acc (III)} &$79.02\%$& $79.80\%$& $\mathbf{81.47\%}$& $80.39\%$\\
        \midrule
        \multicolumn{1}{l}{Loss (IV)} &$0.7293$&$\mathbf{0.6499}$&$0.7215$&$0.6619$ \\
        \multicolumn{1}{l}{Acc (IV)} &$77.70\%$&$80.68\%$&$80.81\%$&$\mathbf{81.64\%}$\\

        % % \multicolumn{1}{l}{Shake-Shake~\cite{devries2017improved}} & $2.56\pm0.07$ & 26.2 & - \\ 
        % \multicolumn{1}{l}{\cite{xie2018snas}} & $16.5$ & $1.98$ & $3.73$ & $2.8$ & $1.5$ \\
        % \multicolumn{1}{l}{\cite{liu2018progressive}} & $15.9$ & $1.83$ & $3.72$ & $3.2$ & $150$ \\
        % \multicolumn{1}{l}{\cite{zoph2018learning}} & $15.8$ & $1.96$ & $3.71$ & $3.3$ & $1800$\\
        % \multicolumn{1}{l}{\cite{liu2018darts}} & $15.8$ & $1.85$ & $3.68$ & $3.4$ & $1$ \\
        % \multicolumn{1}{l}{\cite{real2018regularized}} & $15.9$ & $1.93$ & $3.8$ & $3.2$ & $3150$ \\
        % \multicolumn{1}{l}{\cite{noy2019asap}} & $15.6$ & $1.81$ & $3.73$ & $2.5$ & 0.2 \\ 
        % \midrule
        % \multicolumn{1}{l}{CoNAS} & $15.9$ & 1.44 & $4.11$ & 2.3 & 0.4 \\
        \bottomrule
    \end{tabular}
    % }
    \end{threeparttable}
\end{table}

%\subsection{Performance Comparison on CIFAR-10}\label{ToneEst}

Next, we optimize the five categories of hyperparameters including the learning rate, three convolution layers' and $l_2$ regularization coefficient of the fully connected layer using the architecture and dataset used in the previous section. We train the network using the stochastic gradient descent without a momentum and diminish the learning rate by a factor $0.1$ every 100 epochs. We compare the test loss and accuracy of SH, Hyperband, doubled budgets Random Search with PGSR-HB algorithm. We set the resource, $R = 243$ and the discard ratio input $\eta=3$.  Training epochs is the same for all the algorithms except Random Search with times more training epochs. 
% Setting the total budget of four cycles of Hyperband and PGSR-HB as the baseline, Random Search 2x evaluates 288 randomly sampled hyperparameter configurations with the resource $R$ and SH cycles 24 times as one Hyperband contains six subroutine SH. 
We run each algorithm in Table~\ref{table:CNNPerformance} for four different trials.
% Due to the randomness of hyperparameter optimization algorithms, we compare four different trials of each algorithms as shown in Table~\ref{table:CNNPerformance}. 
The results verify the effectiveness of reducing the hyperparameter space through PGSR as the new algorithm returns better performance for most of the trials. Moreover, PGSR-HB finds the optimal hyperparameters with $83\%$ test accuracy, outperforming the other algorithms from all trials. 

\section{Neural Architecture Search}
\label{NAS_alg}

In this section, we propose a a new algorithm referred to as Compressed-based Neural Architecture Search (CoNAS) for searching the CNN models. In addition, we support the performance of CoNAS with theoretical analysis. 
% We first provide background of NAS especially the ingredients to categorize the NAS algorithm. Then, we briefly explain what the one-shot architecture search method is.
Our proposed algorithm is a novel perspective of merging the existing one-shot NAS methodology with the sparse recovery techniques. This boosts the search time to select a final candidate while outperforms existing the state-of-the-art methods. 
% We support our algorithm by providing the experimental results with a comparison to the existing state-of-the-art NAS algorithms. 

\subsection{Background on One-Shot NAS}
\label{subsec: Background on One-Shot NAS}
In this section, we first briefly describe one-shot NAS techniques with respect to the search space, performance estimation strategy, and search strategy~\cite{elsken2018neural}. Please see~\cite{bender2018understanding,li2019random} for more discussion. 

\subsubsection{Search Space}

Following~\cite{zoph2018learning}, the one-shot method searches the optimal cell as the building block (similar to \cite{szegedy2015going}, \cite{he2016deep}, \cite{huang2017densely}) to construct the final architecture. A cell is a directed acyclic graph (DAG) consisting of $N$ nodes. A j\textsuperscript{th} node, $n^{(j)}$ where $j \in [N]$, has directed edges $(i, j)$ from $n^{(i)}$ where $i < j$ and $i \in [N]$ such that $(i, j)$ transforms the node $n^{(i)}$ to $n^{(j)}$. Let $\mathcal{O}$ is the operation set (e.g. 3x3 max-pool, 3x3 average-pool, 3x3 convolution, identity in convolutional network), and there exists $|\mathcal{O}|$ direct edges between $n^{(i)}$ and $n^{(j)}$. Let $f^{(i, j)}$ maps $n^{(i)}$ to $n^{(j)}$ by summing all transformation of $n^{(i)}$ defined in $\mathcal{O}$. Each intermediate nodes (corresponding to latent representation in general) is computed by the addition of all transformation from predecessor nodes:
\begin{align}
\label{eq: intermediate nodes}
n^{(j)} = \sum_{i < j} f^{(i, j)}(n^{(i)})
\end{align}
As each nodes are connected from previous nodes by summing all possible operations in its operation set $\mathcal{O}$, one-shot NAS literature also describe this as ``{weight-sharing}" methods.

\subsubsection{Performance Estimation Strategy}

The one-shot NAS models the surrogate function which takes the network architecture encoded vectors and returns the estimated performance of the sub-network. \cite{li2019random} and \cite{bender2018understanding} adopts the weight-sharing paradigm following the \eqref{eq: intermediate nodes}. Let $\bm{\alpha} \in \{-1, 1\}^n$ be the architecture encoder string where $n$ is the total number of edges in the DAG. The one-shot surrogate model, the function $f: \{-1, 1\}^n \rightarrow \mathbb{R}$, which trained only once estimates the performance of each sub-network without training individually. While early methods including \cite{zoph2018learning} required expensive computational budget, weight-sharing paradigm significantly increased the search efficiency. 

We summarize the one-shot NAS training protocol from \cite{li2019random}, the training algorithm we adopted for training the super-network (surrogate model). \cite{li2019random} suggests the simple super-network training protocol as shown in Algorithm~\ref{alg: RSWS Training}.

\begin{algorithm}[!ht]
\caption{\textsc{pseudo code of One-Shot Model Training from \cite{li2019random}}}
\label{alg: RSWS Training}
\begin{algorithmic}[1]
\While{not converged}
    \State Randomly sample the architecture from DAG.
    \State Calculate the backward pass for a given minibatch.
    \State \parbox[t]{0.9\linewidth-\algorithmicindent}{Update the weights only corresponds to the edges activated from randomly sampled architecture.}
    %\vspace{1mm}
\EndWhile
\end{algorithmic}
\end{algorithm}

\subsubsection{Search Strategy}

Both \cite{bender2018understanding} and \cite{li2019random} explores the optimal architecture based on random search. While vanilla random search trains each candidates until its convergence (which requires exhaustive computation), one-shot model estimates the performance of candidate architecture by one forward pass evaluation. This allows to explore immense number of candidates than the vanilla random search given equivalent computational budgets. 

\subsection{Proposed Algorithm}

Our proposed algorithm, Compressive sensing-based Neural Architecture Search (CoNAS), infuses ideas from learning a sparse graph (Boolean Fourier analysis) into one-shot NAS. CoNAS consists of two novel components: an expanded search space, and a more effective search strategy. Figure~\ref{fig:overview of CoNAS} shows the overview idea of CoNAS.

Our first ingredient is an expanded search space. Following the approach of DARTS~\cite{liu2018darts}, we define a directed acyclic graph (DAG) where all predecessor nodes are connected to every intermediate node with all possible operations. We represent any sub-graph of the DAG using a binary string $\bm{\alpha}$ called the \emph{architecture encoder}. Its length is the total number of edges in the DAG, and a $1$ (resp. $-1$) in $\bm{\alpha}$ indicates an active (resp. inactive) edge. 

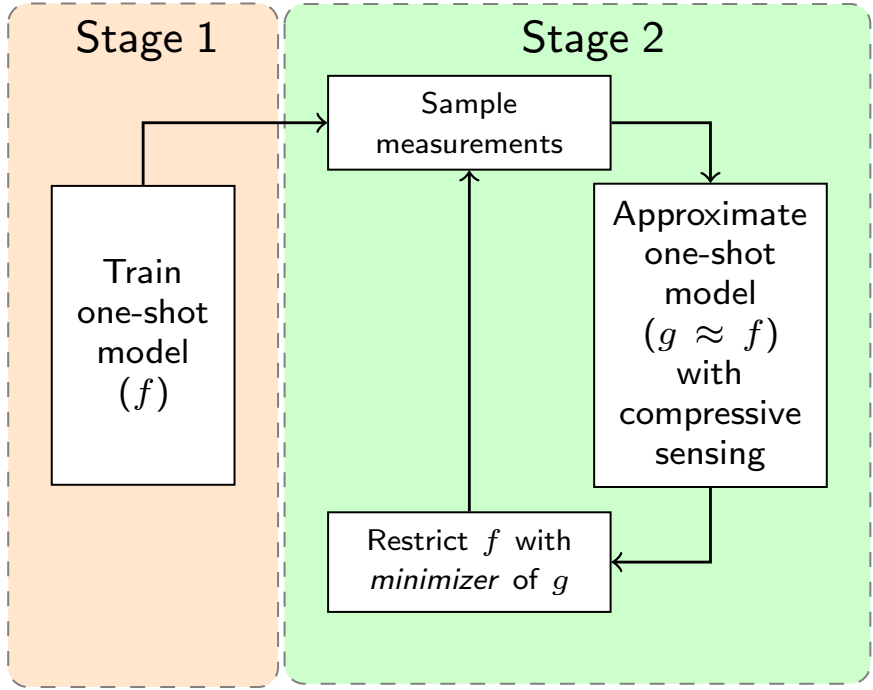
\begin{figure}[!t]
    \centering
    %\resizebox{\linewidth}{!}{%neurips
\resizebox{0.75\linewidth}{!}{
\begin{tikzpicture}[font=\sffamily]
    \fontfamily{cmss}
    \tikzstyle{every node}=[font=\scriptsize]
    \node (oneShot) [oneShotModel] {Train \\ one-shot model \\($f$)};
    
    \path (oneShot)+(2.3, 1.5) node (sample) [sampleNode] {\tiny Sample \\measurements};
    \path (oneShot)+(4, 0) node (CS) [CSNode] {Approximate one-shot model ($g\approx f$) with compressive sensing};
    \path (oneShot)+(2.3, -1.6) node (restriction) [restrictionNode] {\tiny Restrict $f$ with \emph{minimizer} of $g$};
    
    \path [draw, ->, line width = 0.2mm] (oneShot.north) to (0, 1.5) to (sample.west); 
    \path [draw, ->, line width = 0.2mm] (sample.east) to (4, 1.5) to (CS.north);
    \path [draw, ->, line width = 0.2mm] (CS.south) to (4, -1.6) to (restriction.east);
    \path [draw, ->, line width = 0.2mm] (restriction.north) to (sample.south);
    \backgroundc{oneShot}{oneShot}{oneShot}{oneShot}{orange}{Stage 1};
    \backgroundd{sample}{sample}{CS}{restriction}{green}{Stage 2};

\end{tikzpicture}
}
  %  \vspace{1.0em}
    \caption{\sl Overview of CoNAS. A one-shot neural network model $f$ is pre-trained, and an appropriate sub-graph of $f$ is chosen by applying a sparse recovery technique. Iterative sparse recoveries allow to find the larger sub-graph from $f$.}
    \label{fig:overview of CoNAS}
\end{figure}

\begin{figure*}[!t]
\centering
    \resizebox{0.9\textwidth}{!}{
    \resizebox{\linewidth}{!}{
\begin{tikzpicture}
    \fontfamily{cmss}
    \tikzstyle{every node}=[font=\normalsize] %footnotesize
    \node (output) [operation] {Output};
    \path (output.north)+(0, 0.7) node (soft) [operation] {Softmax};
    \path (soft.north)+(0, 0.7) node (cell3) [normalCell] {Cell3};
    \path (cell3.north)+(0, 0.7) node (cell2) [normalCell] {Cell2};
    \path (cell2.north)+(0, 0.7) node (cell1) [normalCell] {Cell1};
    \path (cell1.north)+(0, 0.7) node (stem2) [operation] {Stem2};
    \path (stem2.north)+(0, 0.7) node (stem1) [operation] {Stem1};
    \path (stem1.north)+(0, 0.7) node (image) [operation] {Image};
    
    \path [draw, ->, line width = 0.4mm] (soft.south) -- node [above] {} (output);
    \path [draw, ->, line width = 0.4mm] (cell3.south) -- node [above] {} (soft);
    \path [draw, ->, line width = 0.4mm] (cell2.south) -- node [above] {} (cell3);
    \path [draw, ->, line width = 0.4mm] (cell1) to [out=340, in=20] (cell3);
    \path [draw, ->, line width = 0.4mm] (cell1.south) -- node [above] {} (cell2);
    \path [draw, ->, line width = 0.4mm] (stem2) to [out=340, in=20] (cell2);
    \path [draw, ->, line width = 0.4mm] (stem2.south) -- node [above] {} (cell1);
    \path [draw, ->, line width = 0.4mm] (stem1) to [out=340, in=20] (cell1);
    \path [draw, ->, line width = 0.4mm] (image) to [out=340, in=20] (stem2);
    \path [draw, ->, line width = 0.4mm] (image.south) -- node [above] {} (stem1);
    
    % \node () [draw] at (output.north)+(0, 3.0) {Hello};
    % \path (image.east)+(4, 0.0) node (newcell) [normalCell] {New Cell};
    \draw[dashed] (image.east)+(0.8, 0.3) -- +(0.8, -7.0); 
    
    % Second Plot for the cell./
    \path (image.east) + (2.0, 0) node (previousCell1) [operation] {Cell$_{k-2}$};
    \path (image.east) + (7.3, 0) node (previousCell2) [operation] {Cell$_{k-1}$};
    \path (previousCell1.south) + (0, -1.5) node (choice1) [choice] {Choice1};
    \path (previousCell1.south) + (1.7, -1.5) node (choice2) [choice] {Choice2};
    \path (choice1.south) + (1, -1.0) node (node1) [customNode] {Node1};
    \path (previousCell1.south) + (3.5, -2.5) node (choice3) [choice] {Choice3};
    \path (previousCell1.south) + (5.2, -2.5) node (choice4) [choice] {Choice4};
    \path (node1.south) + (0, -0.7) node (choice5) [choice] {Choice5};
    \path (choice1.south) + (4, -2.5) node (node2) [customNode] {Node2};
    \path (image.east) + (4.6, -6) node (concat) [operation] {Concat.};
    
    \path [draw, ->, line width = 0.4mm] (previousCell1.south) to [out=270, in=90] (choice1);
    \path [draw, ->, line width = 0.4mm] (previousCell1) to [out=320, in=100] (choice3);
    \path [draw, ->, line width = 0.4mm] (previousCell2.south) to [out=270, in=90] (choice4);
    \path [draw, ->, line width = 0.4mm] (previousCell2) to [out=240, in=30] (choice2);
    \path [draw, dashed, ->, line width = 0.4mm] (choice1) to [out=260, in=120] (node1);
    \path [draw, dashed, ->, line width = 0.4mm] (choice2) to [out=280, in=60] (node1);
    \path [draw, dashed, ->, line width = 0.4mm] (choice3) to [out=260, in=110] (node2);
    \path [draw, dashed, ->, line width = 0.4mm] (choice4) to [out=280, in=70] (node2);
    \path [draw, ->, line width = 0.4mm] (node1) to [out=270, in=90] (choice5);
    \path [draw, dashed, ->, line width = 0.4mm] (choice5) to [out=330, in=140] (node2);
    \path [draw, ->, line width = 0.4mm] (node1) to [out=180, in=90] (2.4, 2.7) to [out=270, in=140] (concat.west); %(2.5, 3)
    \path [draw, ->, line width = 0.4mm] (node2) to [out=230, in=30] (concat.east);
    \draw [->, line width = 0.4mm] (concat.south) to (5.29, -0.2);
    
    % \path [draw, dashed, ->, line width = 0.3mm] (previousCell2.south) to [out=270, in=50] (node1);
    % \path [draw, ->, line width = 0.3mm] (node1.south) to [out=240, in=150] (concat);

    \backgrounda{choice1}{choice1}{choice4}{concat}{orange}{\textbf{Cell}}; % West North East South
    \draw[dashed] (previousCell2.east)+(0.40, 0.3) -- +(0.40, -7.0);
    
    % Third Plot for the Choice Block
    \path (previousCell2.east) + (1.5, 0) node (input1) [operation] {Input 1};
    \path (input1) + (1.5, 0) node (input2) [operation] {Input 2};
    \path (input2) + (1.5, 0) node (input3) [operation] {Input 3};
    \path (input1) + (9.0, 0) node (inputL) [operation] {Input L};
    \path (input1) + (4.5, -1.8) node (concat2) [operation] {Concat.};
    \path (concat2) + (0, -1.5) node (identity) [conv] {Identity};
    \path (concat2) + (-2, -1.5) node (conv5x5_1) [conv] {5x5};
    \path (concat2) + (-2, -2.5) node (conv5x5_2) [conv] {5x5};
    \path (concat2) + (-4, -1.5) node (conv3x3_1) [conv] {3x3};
    \path (concat2) + (-4, -2.5) node (conv3x3_2) [conv] {3x3};
    \path (concat2) + (2, -1.5) node (maxPool) [conv] {Max Pool};
    \path (concat2) + (4, -1.5) node (avgPool) [conv] {Avg Pool};
    \path (input1) + (4.5, -6) node (sum) [operation] {Sum};
    
    \draw [->, line width = 0.4mm] (input1.south) to [out=270, in = 180] (concat2) ;
    \path [draw, ->, line width = 0.4mm] (input2.south) to [out=270, in = 180] (concat2);
    \path [draw, ->, line width = 0.4mm] (input3.south) to [out=270, in = 180] (concat2);
    \path [draw, ->, line width = 0.4mm] (inputL.south) to [out=270, in = 0] (concat2);
    
    \draw [dashed, ->, line width = 0.4mm] (concat2) to [out=200, in=40] node [above] {\normalsize $\alpha_{i}$} (conv3x3_1);
    \path [draw, ->, line width = 0.4mm] (conv3x3_1) to [out=270, in=90] (conv3x3_2);
    \path [draw, dashed, ->, line width = 0.4mm] (concat2) to [out=220, in=60] node [below, pos=0.4] {\normalsize $\alpha_{i+1}$} (conv5x5_1);
    \path [draw, ->, line width = 0.4mm] (conv5x5_1) to [out=270, in=90] (conv5x5_2);
    \draw [dashed, ->, line width = 0.4mm] (concat2) to [out=270, in=90] node [] {\normalsize $\alpha_{i+2}$} (identity);
    \path [draw, dashed, ->, line width = 0.4mm] (concat2) to [out=320, in=130] node [above, pos=0.8] {\normalsize $\alpha_{i+3}$} (maxPool);
    \path [draw, dashed, ->, line width = 0.4mm] (concat2) to [out=340, in=140] node [above] {\normalsize $\alpha_{i+4}$} (avgPool);
    \path [draw, dashed, ->, line width = 0.4mm] (conv3x3_2) to [out=270, in=180] (sum);
    \path [draw, dashed, ->, line width = 0.4mm] (conv5x5_2) to [out=270, in=170] (sum);
    \path [draw, dashed, ->, line width = 0.4mm] (identity) to [out=270, in=90] (sum);
    \path [draw, dashed, ->, line width = 0.4mm] (maxPool) to [out=270, in=10] (sum);
    \path [draw, dashed, ->, line width = 0.4mm] (avgPool) to [out=270, in=0] (sum);
    \draw [->, line width = 0.4mm] (sum.south) to (14.675, -0.2);
    
    \backgroundb{conv3x3_1}{concat2}{avgPool}{sum}{green}{\textbf{Choice $i$}};
\end{tikzpicture}}}
%    \vspace{5mm}
    \caption{\sl Diagram inspired by~\cite{bender2018understanding}. The example architecture encoder $\bm{\alpha}$ samples the sub-architecture for $N=5$ nodes (two intermediate nodes) with five different operations. Each component in $\bm{\alpha}$ maps to the edges one-to-one in all \emph{Choice} blocks in a cell. If a bit in $\bm{\alpha}$ corresponds to $1$, the edge activates, while $-1$ turns off the edge. Since the CNN search space finds both \emph{normal cell} and \emph{reduce cell}, the length of $\bm{\alpha}$ is equivalent to $(2+3) \cdot 5 \cdot 2 = 50$.} 
    \label{fig:diagram}
\end{figure*}
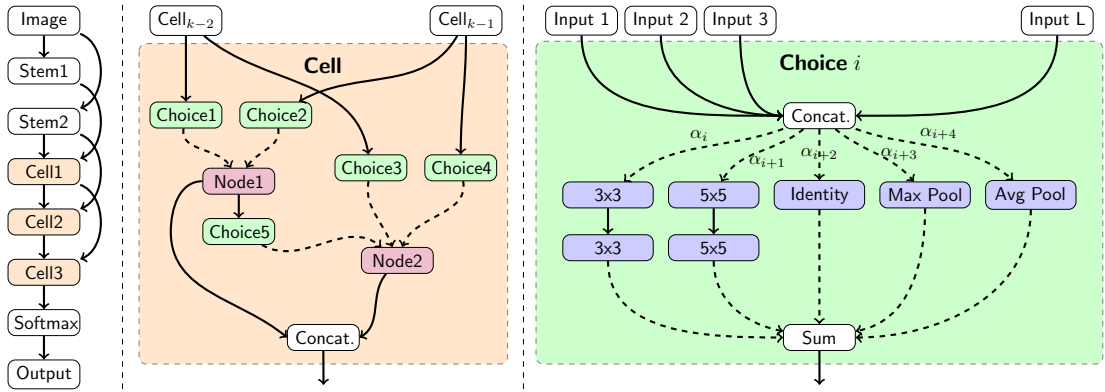

Figure~\ref{fig:diagram} gives an example of how the architecture encoder $\bm{\alpha}$ samples the sub-architecture of the fully-connected model in case of a convolutional neural network. The goal of CoNAS is to find the ``best'' encoder $\bm{\alpha^*}$, which is "close enough" to the global optimum returning the best validation accuracy by constructing the final model with $\bm{\alpha^*}$ encoded sub-graph.

\begin{figure}[!t]
\centering
\subfloat[Normal Cell]{\includegraphics[width=\linewidth]{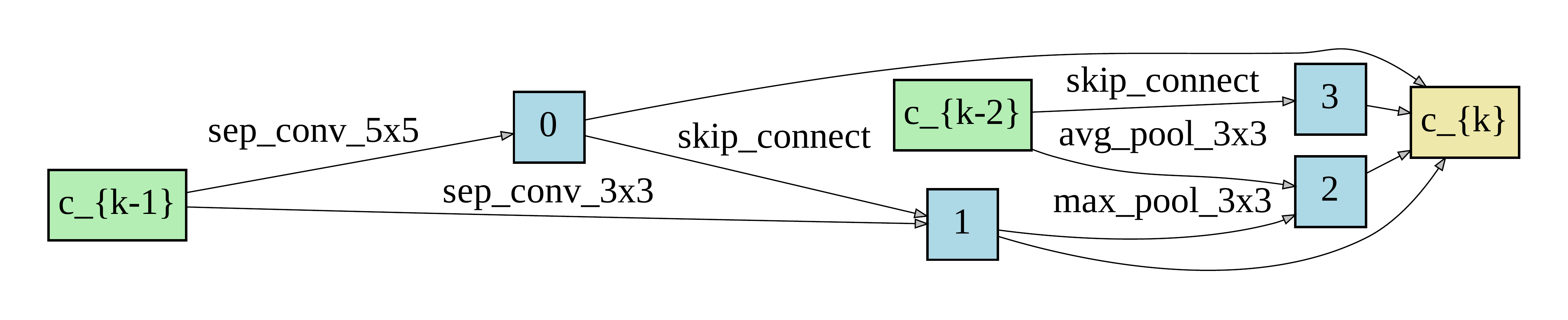}}
\vfil
\subfloat[Reduce Cell]{\includegraphics[width=\linewidth]{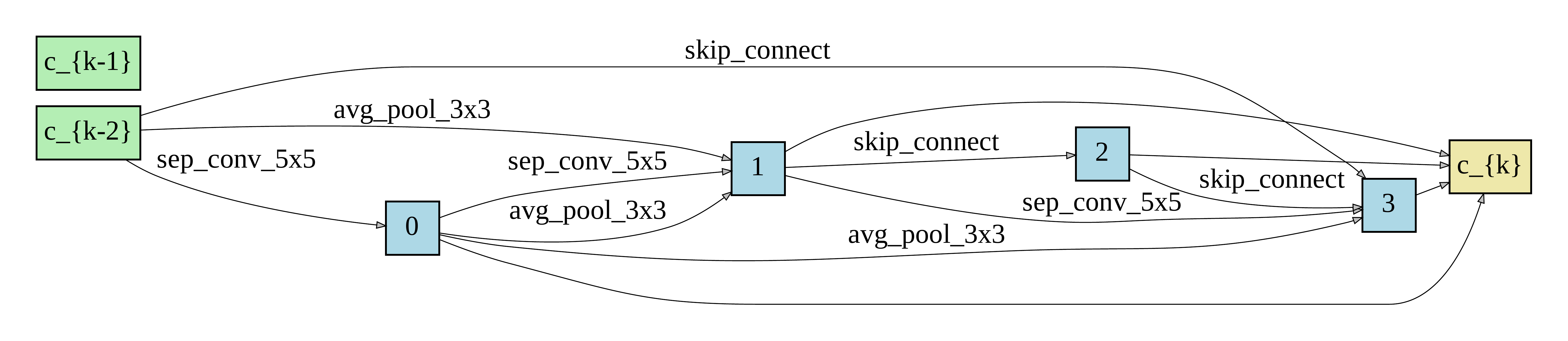}}
%\vspace{1.5em}
\caption{Convolution Cell found from CoNAS. The reduce cell found from CoNAS have a missing connection between $c_{k-1}$ and intermediate nodes which is a valid architecture in our search space.}
\label{fig:cnn cell}
\end{figure}
Since each edge can be switched on and off independently, the proposed search space allows exploring a cell with more diverse connectivity patterns than DARTS~\cite{liu2018darts}. Moreover, the number of possible configurations %from our search space
exceeds similar previously proposed search spaces with constrained wiring rules \cite{li2019random,pham2018efficient,zoph2018learning,real2018regularized}. 

We propose a compressive measuring strategy to approximate the one-shot model with a Fourier-sparse Boolean function. Let $f: \{-1, 1\}^n \rightarrow \mathbb{R}$ map the sub-graph of the one-shot pre-trained model encoded by $\bm{\alpha}$ to its validation performance.  Similar to \cite{hazan2017hyperparameter}, % flag
we collect a small number of function evaluations of $f$, and reconstruct the Fourier-sparse function $g \approx f$ via sparse recovery algorithms with randomly sampled measurements.
%with performance estimation strategy} (Section~\ref{sec:search strategy and PES}) \textbf{based on one-shot neural architecture}  %(2\textsuperscript{nd} contribution). 
Then, we solve $\argmin_{\bm{\alpha}}g(\bm{\alpha})$ by exhaustive enumeration over all coordinates in its sub-cube $\{-1, 1\}^{\overline{J}}$ where $(J, \overline{J})$ partitions $[n]$ (Definition~\ref{def:restriction})\footnote{This is similar to the idea of de-biasing in the \emph{Hard-Thresholding (HT)} algorithm~\cite{foucart2017mathematical} where the support is first estimated, and then within the estimated support, the coefficients are calculated through least-squares estimation.}. If the solution of the $\argmin_{\bm{\alpha}}g(\bm{\alpha})$ does not return enough edges to construct the cell (some intermediate nodes are disconnected), we simply connect the intermediate nodes to the previous cell output, $\text{Cell}_{k-2}$, using the Identity operation (this does not increase neither the model size nor number of multiply-add operations). Larger cells can be found from multiple iterations by restricting the approximate function $g$ and with fixing the bit values found in the previous solution, and randomly sampling sub-graphs in the remaining edges. 

\begin{algorithm}[h!]
  \caption{\textsc{Pseudocode of CoNAS}}
  \label{alg: CoNAS}
  \begin{algorithmic}[1]
      \State\textbf{Inputs:} Number of one-shot measurements $m$, stage $t$, sparsity $s$, lasso parameter,  $\lambda$, Bernoulli $p$
      \phase{Training the One-Shot Model}
        \Procedure{Model Training}{}
        \While{not converged}
            \State \parbox[t]{0.9\linewidth-\algorithmicindent}{
            Randomly sample a sub-architecture encoded binary vector $\bm{\alpha}$ according to Bernoulli($p$)}
            %\vspace{0.2em}
            \State \parbox[t]{0.9\linewidth-\algorithmicindent}{ 
            Update $w_{\bm{\alpha}}$ by descending $\nabla_{w_{{\bm{\alpha}}}} \mathcal{L}_{train}(w_{{\bm{\alpha}}})$}
            % \vspace{0.3em}
        \EndWhile
        \EndProcedure
    \phase{Search Strategy}
    \Procedure{One-Shot Model Approximation via Compressive Sensing}{}
        \For{$k \in \{1, \ldots, t\}$}
            \State Collect $\mathbf{y} = (f(\bm{\alpha}_1), f(\bm{\alpha}_2), \ldots, f(\bm{\alpha}_m))^\top$.
            \State Solve $$\mathbf{x}^* = \argmin_{\mathbf{x}} \|\mathbf{y} - \mathbf{A} \mathbf{x}\|_2^2 + \lambda\|\mathbf{x}\|_1$$
            \State \parbox[t]{0.9\linewidth-\algorithmicindent}{
            Let $x_1^*, x_2^*, \ldots, x_s^*$ be the $s$ absolutely largest coefficients of $\mathbf{x}^*$.  Construct}
          %  and $\chi_1, \chi_2, \ldots,$ \par
          % \hspace{2.4mm} $\chi_s$ are the Boolean basis corresponding to the coefficient respectively. Let $g$ be \par 
            %\hspace{2.4mm} a Fourier series with $s$ monomials: 
            $$g(\bm{\alpha}) = \sum_{i=1}^{s} x_i^* \chi_i(\bm{\alpha})$$
            \State \parbox[t]{0.9\linewidth-\algorithmicindent}{
            Compute minimizer $\mathbf{z} = \argmin_{\bm{\alpha}} g(\bm{\alpha})$ and let $\overline{J}$ the set of indices of $z$}
            %\vspace{0.2em}
            \State $f = f_{J|\mathbf{z}}$
        \EndFor
        \State \parbox[t]{0.9\linewidth-\algorithmicindent}{ 
        Construct the cell by activating the edge where $z_i = 1$ where $i \in [n]$.}
        %\vspace{0.2em}
    \EndProcedure
  \end{algorithmic}
\end{algorithm}

We now describe CoNAS in detail, with pseudocode provided in Algorithm~\ref{alg: CoNAS}. 
We first train a one-shot model with standard backpropagation but only updates the weight corresponding to the randomly sampled sub-graph edges for each minibatch.
% We first train a weight-shared model using the training algorithm of random search with weight-sharing (RSWS)~\cite{li2019random} \textbf{WHICH?}. 
Then, we randomly sample sub-graphs by generating architecture encoder strings $\bm{\alpha} \in \{-1, 1\}^n$ using a $Bernoulli(p)$ distribution for each bit of $\bm{\alpha}$ independently (We set $p = 0.5$).

In the second stage, we collect $m$ measurements of randomly sampled sub-architecture performance denoted by $\mathbf{y} = (f(\bm{\alpha}_1), f(\bm{\alpha}_2), \ldots, f(\bm{\alpha}_m))^T$.  Next, we construct the \emph{graph-sampling matrix} $\mathbf{A} \in \{-1, 1\}^{m \times |\mathcal{P}_d|}$ with entries 
\begin{align}
\label{eq:sampling matrix}
% \mathbf{A}_{l, k} = \chi_{S}(\bm{\alpha}_l), \:\:\:\:\:\:\: l \in [m], k \in [|\mathcal{P}_d|], S \subseteq [n], |S| \leq d,
\mathbf{A}_{l, k} = \chi_{S_k}(\bm{\alpha}_l), \:\:\:\:\:\:\: l \in [m], k \in [|\mathcal{P}_d|], S \subseteq [n], |S| \leq d,
\end{align}
where $d$ is the maximum degree of monomials in the Fourier expansion, and $S_k$ is the index set corresponding to k\textsuperscript{th} Fourier basis. We solve the familiar Lasso problem~\cite{tibshirani1996regression}:
\begin{align}
\label{eq:Lasso}
    \mathbf{x}^* = \argmin_{\mathbf{x} \in \mathbb{R}^{|\mathcal{P}_d|}} \|\mathbf{y} - \mathbf{A} \mathbf{x}\|_2^2 + \lambda\|\mathbf{x}\|_1, 
\end{align}
to (approximately) recover the global optimizer $\mathbf{x}^*$, the vector contains the Fourier coefficients corresponding to $\mathcal{P}_d$. We define an approximate function $g \approx f$ with Fourier coefficients with the top-$s$ (absolutely) largest coefficients from $\mathbf{x}^*$, and compute $\bm{\alpha}^* = \argmin_{\bm{\alpha}} g(\bm{\alpha})$, resulting all the possible points in the subcube defined by the support of $g$ (this computation is feasible if $s$ is small). Multiple stages of sparse recovery (with successive restrictions to previously obtained optimal $\bm{\alpha}^*$) enable us to approximate additional monomial terms. 
% This is similar to the idea of de-biasing tep in \emph{Hard-Thresholding (HT)} algorithm~\cite{foucart2017mathematical} introduced in the sparse recovery literature where first the support is estimated and then within the estimated support the sparse coefficients are calculated through a least-square estimation of the sparse Fourier coefficient. 
Finally, we obtain a cell to construct the final architecture by activating the edges corresponding to all ${i \in [n]}$ such that $\alpha_i^* = 1$. 
% We include theoretical support for CoNAS in Appendix~\ref{sec:theoretical support}.

\begin{table*}[!h]
    \captionsetup{width=.80\linewidth}
    \caption{\sl \textbf{Comparison with hand-designed networks and state-of-the-art NAS methods on CIFAR-10} (Lower test error is better). 
    % The results are grouped as follows: manually designed networks, published NAS algorithms, and our experimental results. 
    The average test error of our experiment used five random seeds. Table entries with "-" indicates that either the field is not applicable or unknown. The methods listed in this table are trained with auxiliary towers and cutout augmentation. 
    % Running time cost is measured on NVIDIA TITAN X GPU. 
    % The reported time of CoNAS includes both training one-shot model and gathering measurements for the sparse recovery.
    } %\vspace{1em}
    \centering
    \begin{threeparttable}
    % \resizebox{0.9\linewidth}{!}{
    \begin{tabular}{c c c c c c}
        \toprule
        \multicolumn{1}{c}{} & \multicolumn{1}{c}{\textbf{Test Error}} & \textbf{Params} & \textbf{Multi-Add}& \textbf{Search} \\
        \multicolumn{1}{l}{\textbf{Architecture}} & (\%) & (M) & (M) & GPU days \\ 
        \midrule
        % \multicolumn{1}{l}{Shake-Shake~\cite{devries2017improved}} & $2.56\pm0.07$ & 26.2 & - \\ 
        \multicolumn{1}{l}{PyramidNet~\cite{yamada2018shakedrop}} & $2.31$ &26 & - & -\\
        \multicolumn{1}{l}{AutoAugment~\cite{cubuk2018autoaugment}} & $1.48$ & $26$ & - & - & \\%[0.5ex]
        \midrule
        \multicolumn{1}{l}{ProxylessNAS~\cite{cai2018proxylessnas}} & $2.08$ & 5.7 & - & 4  \\
        \multicolumn{1}{l}{NASNet-A~\cite{zoph2018learning}} & $2.65$ & 3.3 & - & 2000 \\
        \multicolumn{1}{l}{AmoebaNet-B~\cite{real2018regularized}} & $2.55\pm0.05$ & 2.8 & - & 3150 \\
        \multicolumn{1}{l}{GHN\textsuperscript{+}~\cite{zhang2018graph}} & $2.84 \pm 0.07$ & 5.7 & - & 0.84 \\
        \multicolumn{1}{l}{SNAS~\cite{xie2018snas}} & $2.85 \pm 0.02$ & 2.8 & - & 1.5  \\
        \multicolumn{1}{l}{ENAS~\cite{pham2018efficient}}& 2.89 & 4.6 & - & 0.45 \\
        % \multicolumn{1}{l}{DARTs (first order)~\cite{liu2018darts}} & - & $3.00 \pm 0.14$ & 3.3 & 1.5 & gradient \\
        \multicolumn{1}{l}{DARTs~\cite{liu2018darts}} & $2.76\pm0.09$ & 3.3 & 548 & 4 \\
        \multicolumn{1}{l}{Random Search~\cite{liu2018darts}} & $3.29 \pm 0.15$ & 3.1 & - & 4 \\
        \multicolumn{1}{l}{ASHA~\cite{li2019random}} & $2.85|3.03\pm0.13$ & 2.2 & - & - \\
        \multicolumn{1}{l}{RSWS~\cite{li2019random}} & $2.71|2.85\pm0.08$ & 3.7 & 634 & 2.7 \\
        \multicolumn{1}{l}{DARTs\textsuperscript{\#}~\cite{li2019random}} &$2.62|2.78 \pm 0.12$ & 3.3 & - & 4 \\
        \midrule
        \multicolumn{1}{l}{DARTs\textsuperscript{\dag}} & $2.59|2.78\pm0.13$ & 3.4 & 576 & 4 \\ 
        % \multicolumn{1}{l}{RSWS\textsuperscript{*}} & 2.47 & $2.59\pm0.11$&4.7 & 0.7 \\
        \multicolumn{1}{l}{\textbf{CoNAS (t=1)}} & $\mathbf{2.57}|2.74\pm0.12$ & \textbf{2.3} & \textbf{386} & \textbf{0.4} \\
        \multicolumn{1}{l}{\textbf{CoNAS (t=4)}} & $2.55|2.62\pm0.06$ & 4.8 & 825 & 0.5 \\
        \multicolumn{1}{l}{\textbf{CoNAS (t=1, C=60)}\textsuperscript{+}+AutoAugment} & $1.87$ & 6.1 & 1019 & 0.4 \\
        \bottomrule
    \end{tabular}
    % }
    \begin{tablenotes}
    \small
    \item \textsuperscript{\#} DARTS experimental results from~\cite{li2019random}.
    \item \textsuperscript{\dag} Used DARTS search space with five operations for direct comparisons.
    % \item \textsuperscript{*} Used the search space defined in Section~\ref{subsec:search space CoNAS}. Used five operations for direct comparison.
    \item \textsuperscript{+} `C' stands for the number of initial channels. Trained 1,000 epochs with AutoAugment.
   % \vspace{-3mm}
    \end{tablenotes}
    \end{threeparttable}
    \label{table: CIFAR10 Comparison Table}
\end{table*}

\subsection{Theoretical Analysis for CoNAS.} 

The system of linear equations  $\mathbf{y} = \mathbf{Ax}$ with the graph-sampling matrix $\mathbf{A} \in \{-1, 1\}^{m \times O(n^d)}$, measurements $\mathbf{y} \in \mathbb{R}^m$, and Fourier coefficient vector $\mathbf{x} \in \mathbb{R}^{O(n^d)}$ is an ill-posed problem when $m \ll O(n^d)$ for large $n$. However, if the graph-sampling matrix satisfies \emph{Restricted Isometry Property (RIP)}, the sparse coefficients, $\mathbf{u}$ can be recovered: 
\begin{definition}
A matrix $\mathbf{A}\in\mathbb{R}^{m\times \mathcal{O}(n^d)}$ satisfies the restricted isometry property of order $s$ with some constant $\delta$ if for every $s$-sparse vector $\mathbf{u}\in\mathbb{R}^{\mathcal{O}(n^d)}$ (i.e., only $s$ entries are non-zero) the following holds:
$$(1-\delta) \|\mathbf{u}\|_2^2\leq\|\mathbf{Au}\|_2^2\leq(1+\delta) \|\mathbf{u}\|_2^2.$$
\end{definition}

We defer the history of improvements on the upper bounds of the number of rows from bounded orthonormal dictionaries (matrix $\mathbf{A}$) for which $\mathbf{A}$ is guaranteed to satisfy the restricted isometry property with high probability in Appendix~\ref{appendix: prior works CS}. To the best of our knowledge, the best known result with mild dependency on $\delta$ (i.e., $\delta^{-2}$) is due to~\cite{haviv2017restricted}, which we can apply for our setup. It is easy to check that the graph-sampling matrix $\mathbf{A}$ in our proposed CoNAS algorithm satisfies BOS for $K=1$ (Eq~\ref{eq:sampling matrix}).

\begin{theorem}
\label{thm:recovery on main}
Let the graph-sampling matrix $A\in\{-1,1\}^{m\times \mathcal{O}(n^d)}$ be constructed by taking $m$ rows (random sampling points) uniformly and independently from the rows of a square matrix $\mathbf{M}\in\{-1,1\}^{\mathcal{O}(n^d)\times \mathcal{O}(n^d)}$. Then the normalized matrix $\mathbf{A}$ with $m = \mathcal{O}(\log^2(\frac{1}{\delta})\delta^{-2}s\log^2(\frac{s}{\delta})d \log(n))$ with probability at least $1-2^{-\Omega(d\log n\log(\frac{s}{\delta}))}$ satisfies the restricted isometry property of order $s$ with constant $\delta$; as a result, every $s$-sparse vector $\mathbf{u}\in \mathbb{R}^{\mathcal{O}(n^d)}$ can be recovered from the sample $y_i$'s:
\begin{align}
\mathbf{y} = \mathbf{Au} = \big(\sum_{j=1}^{|\mathcal{O}(n^d)|} u_j\mathbf{A}_{i,j}\big)_{i=1}^{m},
\end{align}
by LASSO (equation~\ref{eq:Lasso}).
\end{theorem}

\begin{proof}
First, we note that the graph-sampling matrix $A$ is a BOS matrix with $K=1$; hence, directly invoking Theorem 4.5 of \cite{haviv2016list} to our setting, we can see that matrix $A$ satisfies RIP. Now according to Theorem 1.1 of \cite{candes2008restricted}, letting $\delta<\sqrt{2} -1$, the $l_1$ minimization or LASSO will recover exactly the $s$ sparse vector $u$. For instance, in our experiments, we have selected $m=1000$ which is consistent with our parameters, $d=2, s=10, n =140$. 
\end{proof}

Here, it is worthwhile to mention two points: first, the above upper bound on the number of rows of the graph-sampling matrix $A$ is the tightest bound (according to our knowledge) for the BOS matrices to satisfy RIP. There exist series of results establishing the RIP for BOS matrices during the last 15 year. We have reviewed these results in the Appendix~\ref{appendix: prior works CS}. Second, instead of LASSO, one can use any sparse recovery method (such as IHT~\cite{blumensath2009iterative}) in our algorithm. In essence, Theorem~\ref{thm:recovery on main} provides a successful guarantee for recovering the optimal sub-network of a given size given a sufficient number of performance measurements.

% \subsection{Difference to Hyperparameter Optimimization}

% Building upon the approach of \cite{stobbe2012learning}, \cite{hazan2017hyperparameter} develop a spectral approach called \emph{Harmonica} for hyperparameter optimization (HPO) by encoding hyperparameters as binary strings. CoNAS also follows the same path, albeit for NAS. While NAS and HPO are sister meta-learning problems, we emphasize that our focus is exclusively on NAS, while \cite{hazan2017hyperparameter} exclusively focus on HPO. 

% Moreover, the techniques of \cite{hazan2017hyperparameter} cannot be directly applied to the NAS problem. We need to define our search space, encode our search problem in terms of Boolean variables, and propose how to gather measurements. All these are new to our paper: in particular, CoNAS proposes gathering measurements within tractable sampling time as described in Appendix~\ref{appendix: supplementary background of one-shot}, while Harmonica naively gathers the approximated measurements by training the model for each randomly sampled hyperparameter choice. In Appendix~\ref{appendix: supplementary background of one-shot}, we show that the sampling times with the naive sampling method for 1,000 measurements takes approximately 174 GPU days, whereas CoNAS only take 0.02 GPU days to gather 1,000 measurements via one-shot super-network. Finally, Harmonica requires invocation of a baseline hyperparameter optimization method (such as random search, successive halving~\cite{jamieson2016non}, or Hyperband~\cite{li2017hyperband}), which CoNAS does not require.  

\section{NAS Experimental Results}
\label{NAS_exp}

We experiment on image classification NAS problems: CNN search on CIFAR-10, CIFAR-100, Fashion MNIST and SVHN. We describe the training details for CIFAR-10 in Sections~\ref{exp: CoNAS Exp Results CNN}. Our evaluation setup for training the final architecture (CIFAR-10) is the same as that reported in DARTS and RSWS.

\subsection{Convolutional Neural Network}
\label{exp: CoNAS Exp Results CNN}
\textbf{Architecture Search.} We create a one-shot architecture similar to RSWS with a cell containing $N=7$ nodes with two nodes as input and one node as output; our wiring rules between nodes are different and as in Figure~\ref{fig:diagram}. We used five operations: $3 \times 3$ and $5 \times 5$ separable convolutions, $3 \times 3$ max pooling, $3 \times 3$ average pooling, and Identity. % Since CNN cell requires both normal cell and reduce cell, the length of architecture encoder is $(2+3+4+5) \cdot 5 \cdot 2 = 140$.
On CIFAR-10, we equally divide the $50,000$-sample training set to training and validation sets to train one-shot super-network, following \cite{li2019random} and \cite{liu2018darts}. We train a one-shot model by sampling the random sub-graph under $\text{Bernoulli}
(0.5)$ sampling with eight layers and $16$ initial channels for $100$ training epochs. % exposing the model to many sub-architecture as possible. 
%We sampled sub-architecture encoded strings $\bm{\alpha}$ respect to Bernoulli distribution with $p=\frac{1}{2}$ for each digit. 
All other hyperparameters used in training the one-shot model are the same as in \cite{li2019random}. 

%After training the one-shot architecture, we randomly sample the encoded Boolean strings of the sub-graph and collect the proxy validation loss with ten randomly sampled mini-batches from validation data. The weight-shared model exponentially reduces the time to gather the measurements compare to the conventional proxy method discussed in the Section~\ref{sec:search strategy and PES}. 

%Multiple stages of sparse recovery return a sufficient number of essential edges to activate for the final cell.

We run CoNAS in two different settings to find small and large size CNN cells. Specifically, we use the sparsity parameters $s=10$, Fourier basis degree $d=2$, and Lasso coefficient $\lambda=1$ (We include experiments with varying lasso coefficients in Subsection~\ref{exp: Lasso Stability}). As a result, we found the \emph{normal cell} and \emph{reduce cell} with one sparse recovery stage as shown in Appendix~\ref{fig:cnn cell} (the larger CNN cells were found with multiple sparse recovery stages). Repeating four stages ($t=4$) of sparse recovery with restriction in definition~(\ref{def:restriction}) returns an architecture encoder $\bm{\alpha^*}$ with numerous operation edges in the cells (Please see Figure~\ref{fig:cnn cell, t=4} for the found architecture). Now, we evaluate the model found by CoNAS as follows:

% We run CoNAS with multiple sparse recovery stages. Specifically, following \cite{hazan2017hyperparameter}, we used the parameters $s=10$, Fourier basis degree $d=2$, and Lasso coefficient $\lambda=1$ for each stage.  
% %verified the stability of the Boolean function recovery with various Lasso coefficients through experiments. 
% Repeating three to four stages ($t=3,4$) of sparse recovery with restriction returned an architecture encoder $\bm{\alpha^*}$ with sufficient number of edges needed to construct both \emph{normal cell} and \emph{reduce cell} (e.g., see Figure~\ref{fig:cnn cell}). The final architecture found from CoNAS used $t=4$. %\textbf{WHICH t did we use finally}

\textbf{Architecture Evaluation.}
\label{sec: architecture evaluation}
We re-train the final architecture with the learned cell and with the same hyperparameter configurations in DARTS to make the direct comparisons. We use NVIDIA TITAN X, GTX 1080, and Tesla V100 for final architecture training process. We use TITAN X for searching the architecture for our experiment to conduct the fair comparison on search time. CONAS search time in Table~\ref{table: CIFAR10 Comparison Table} includes both training an one-shot model and gathering measurements. CoNAS cells from four sparse recovery stages (t=4) cannot use the same minibatch size (i.e., 96) used in DARTs and RSWS, due to the hardware constraint; instead, we re-train the final model with minibatch size 56 with TITAN X. 

CoNAS architecture with one sparse recovery (t=1) outperforms DARTs and RSWS (\textcolor{black}{stronger than vanilla random search}) in test errors with smaller parameters, multiply-addition operations, and search time. In addition, CoNAS with four recovery stages (t=4) performs better than CoNAS (t=1) on both lower test error average and deviation; however, it requires larger parameters and multiply-add operations compared to DARTs, RSWS, and CoNAS (t=1). 
We also train CoNAS (t=1) with increasing the number of channels from 36 to 60 and training epochs from 600 to 1,000 together with a recent data augmentation technique called AutoAugment~\cite{cubuk2018autoaugment}, which breaks through $2\%$ test error barrier on CIFAR-10.

\begin{table}[!t]
    \centering
    \captionsetup{width=\linewidth}
    \caption{\sl \textbf{Image Classification Test Error of CoNAS on Multiple Datasets}. We compare the performance of CoNAS on different datasets with existing NAS results. The experiment details for CoNAS is described in Appendix~\ref{appendix: training detail}.} %\vspace{1em}
    \label{table: Additional Dataset Table}
    \begin{threeparttable}
    % \resizebox{\linewidth}{!}{
    \begin{tabular}{c c c c c c}
        \toprule
        & \multicolumn{1}{c}{\textbf{CIFAR100}} & \textbf{SVHN} & \textbf{F-MNIST}& \textbf{Params} & \textbf{Search} \\
        \multicolumn{1}{l}{\textbf{Arch}} & (\%) & (\%) & (\%) & (M) & (GPU) \\ 
        \midrule
        % \multicolumn{1}{l}{Shake-Shake~\cite{devries2017improved}} & $2.56\pm0.07$ & 26.2 & - \\ 
        \multicolumn{1}{l}{\cite{xie2018snas}} & $16.5$ & $1.98$ & $3.73$ & $2.8$ & $1.5$ \\
        \multicolumn{1}{l}{\cite{liu2018progressive}} & $15.9$ & $1.83$ & $3.72$ & $3.2$ & $150$ \\
        \multicolumn{1}{l}{\cite{zoph2018learning}} & $15.8$ & $1.96$ & $3.71$ & $3.3$ & $1800$\\
        \multicolumn{1}{l}{\cite{liu2018darts}} & $15.8$ & $1.85$ & $3.68$ & $3.4$ & $1$ \\
        \multicolumn{1}{l}{\cite{real2018regularized}} & $15.9$ & $1.93$ & $3.8$ & $3.2$ & $3150$ \\
        \multicolumn{1}{l}{\cite{noy2019asap}} & $15.6$ & $1.81$ & $3.73$ & $2.5$ & 0.2 \\ 
        \midrule
        \multicolumn{1}{l}{CoNAS} & $15.9$ & 1.44 & $4.11$ & 2.3 & 0.4 \\
        \bottomrule
    \end{tabular}
    % }
    \begin{tablenotes}
    \small
    \item Experimental result of \cite{xie2018snas}, \cite{liu2018progressive}, \cite{zoph2018learning}, \cite{liu2018darts}, \cite{real2018regularized}, \cite{noy2019asap} are taken from \cite{noy2019asap}.
    % \item \textsuperscript{*} Used the search space defined in Section~\ref{subsec:search space CoNAS}. Used five operations for direct comparison.
    % \item \textsuperscript{+} The search space is not comparable to CoNAS.
    % \vspace{-4mm}
    \end{tablenotes}
    \end{threeparttable}
\end{table}

\subsection{Transfer to other datasets}
\label{exp:other dataset}

We test the cell found from CIFAR-10 to evaluate the transferability to different datasets: CIFAR-100, SVHN, and Fashion-MNIST in Table~\ref{table: Additional Dataset Table}. 
%Even though the comparison is not apple to apple due to unknown hyperparameters, 
As we can see, CoNAS achieves the competitive results with the smallest architecture size (equivalent to number of parameter) compared to the other algorithms.

\subsection{Stability on Lasso Parameters}
\label{exp: Lasso Stability}
We check our algorithm's stability on lasso parameter by observing the solution given exact same measurements. Denote $\alpha^*_{\lambda=l}$ as the architecture encoded output from CoNAS given $\lambda=l$. We compare the hamming distance and the test error between $\alpha^*_{\lambda=1}$ and other $\lambda$ values ($\lambda = 0.5, 2, 5, 10$). The average support of the solution from one sparse recovery is 15 out of the 140 length. The average hamming distance between two randomly generated binary strings with $\text{supp}(\alpha^*)=15$ from $100,000$ samples was $27.58\pm1.82$. Our experiment shows a stable performance under various lasso parameters with small hamming distances regards to various $\lambda$. Also we measure the average test error with 150 training epochs on different $\lambda$ values as shown in Table~\ref{table:lasso parameter testing}. For the baseline comparison, we compare CoNAS solutions with the randomly chosen architecture with 15 operations (edges).

\begin{table}[!t]
    \caption{Lasso Parameter Stability Experiment.}
    \label{table:lasso parameter testing}
    %\vspace{1.5em}
    \centering
    \resizebox{0.7\linewidth}{!}{
    \begin{tabular}{c c c c c c}
    \toprule
        Criteria & $\lambda = 0.5$ & $\lambda = 2.0$ & $\lambda = 5.0$ & $\lambda = 10.0$ & Random\\
    \midrule
        \multicolumn{1}{l}{Hamming Dist.} & $0$ & $0$ & $8$ & $12$ & 29\\
        \multicolumn{1}{l}{Test Error (\%)} & $3.74$ & $3.74$ & $3.51$ & $3.62$ & $4.43$ \\
        \multicolumn{1}{l}{Param (M)} & 2.3 & 2.3 & 2.6 & 2.6 & 2.7\\
        \multicolumn{1}{l}{MADD (M)} & 386 & 386 & 455 & 449 & 444\\
    \bottomrule
    \end{tabular}
    }
\end{table}

\subsection{Discussion}
Noticeably, CoNAS achieves improved results on CIFAR-10 in both test error and search cost when compared to the previous state-of-the-art algorithms: DARTs, RSWS, and ENAS. In addition, not only CoNAS finds the cell with smallest parameter size and multiply-add operations than the other NAS approaches, but also it obtains a better test error with $2.57\%$.
%than DARTS, RSWS, and ENAS.  
% We see that RSWS with an equivalent search space to CoNAS obtained a better architecture (with best test error $2.47\%$); however, the standard deviation of test errors over different pseudorandom was greater than CoNAS. 
Many previous NAS papers have focused on the search strategy, while adopted the same search space to~\cite{zoph2018learning} and~\cite{liu2018darts}. Our experimental results highlight the importance of both seeking new performance strategies and the search space. 

% Finally, on PTB, our experiments show that CoNAS finds a better RNN architecture than RSWS, DARTs using an equivalent or less search cost. However, the reported test perplexity of DARTs and RSWS outperforms both valid and test perplexity of CoNAS. %Defining the proper search space for language model suitable to CoNAS will be interesting 
% We could not include true one-to-one comparisons with other algorithms since they used different search spaces. %; but the experimental results show that CoNAS search space does not suitable with Recurrent Highway Network since RHN suggests the rule allowing one nonlinearity transform input for each intermediate node. 

\section{Conclusions}
\label{conc}
In this paper, we considered the problems of hyperparameter optimization and neural architecture search through the lens of compressive-sensing. As our primary contribution, we first extended Harmonica algorithm by introducing the new log-linear representation for numerical hyperparameters by posing group sparsity in hyperparameters space. We support our algorithms by providing some experiments for the classification task and by visualizing the reduction of the hyperparameters space. We also tackled neural architecture search problem and proposed CoNAS, which expresses the surrogate function of the one-shot super-network via Boolean loss function. 
% exploring the optimal architecture from optimizing solution of the surrogate function.
We supported our NAS algorithm with a theoretical analysis and extensive experimental results on convolutional networks. Several interesting future works remain, including applying the boolean function scheme to the other neural network architectures such as recurrent neural network, generative adversarial network, and transformers. Moreover, extending the boolean function idea to the neural network compression scheme (which attempts to prune the weights) by finding the appropriate binary mask which Hadamard product to weights while achieving minor (or no) degrade of performance will be interesting direction of future study.

%and recurrent networks.
% with the theoretical condition to recover the Boolean function.

% or


% \appendices
\appendix
\section{Appendix}
\subsection{Algorithms for SH and Hyperband}
\label{appendix: SH and Hyperband}

In this section, we include the pseudo algorithm including Successive Halving (Algorithm~\ref{alg: Successive halving}) and Hyperband (Algorithm~\ref{alg: Hyperband}).

\begin{algorithm}[!ht]
\caption{Successive Halving (SH) from ~\cite{jamieson2016non}}
\label{alg: Successive halving}
\begin{algorithmic}[1]
    \State \textbf{Input: } Resource $R$, scaling factor $\eta$
    \State \textbf{Initialization: } $s_{max}=\lfloor \log_{\eta}(R)\rfloor$, $B=(s_{max}+1)R$
    \State $n=R$, $r=1$
    \State T = {\fontfamily{pcr}\selectfont sample\_configuration(n)}
    \For{$i \in \{0,\ldots,s_{max}\}$}
        \State $n_i = \lfloor n\eta^{-i}\rfloor$
        \State $r_i = r\eta^i$
        \State $L = \{f(t,r_i): t \in T\}$
        \State $T = \text{top}_k(T,L,\lfloor \frac{n_i}{\eta} \rfloor)$
    \EndFor
    \State \Return{Configuration with the smallest loss}
\end{algorithmic}
\end{algorithm}

\begin{algorithm}[t]
\caption{\textsc{Hyperband from \cite{li2017hyperband}}}
\label{alg: Hyperband}
\begin{algorithmic}[1]
\State\textbf{Inputs:} Resource $R$, scaling factor $\eta$
\State\textbf{Initialization:} $s_{max}=\lfloor \log_{\eta}(R)\rfloor$, $B=(s_{max}+1)R$
\For{$s \in \{s_{max}, s_{max}-1,\ldots,0\}$}
\State $n=\lceil \frac{B}{R} \frac{\eta^s}{(s+1)}\rceil$, $r=R\eta^{-s}$
\State T = {\fontfamily{pcr}\selectfont sample\_configuration(n)}
\For{$i \in \{0,\ldots,s\}$}
\State $n_i = \lfloor n\eta^{-i}\rfloor$
\State $r_i = r\eta^i$
\State $L = \{f(t,r_i): t \in T\}$
\State $T = \text{top}_k(T,L,\lfloor \frac{n_i}{\eta} \rfloor)$
\EndFor
\EndFor
\State\textbf{return} Configuration with the smallest loss
\end{algorithmic}
\end{algorithm}

\subsection{Prior Works on Recovery Conditions on Compressive Sensing}
\label{appendix: prior works CS}
There has been significant research during the last decade in proving upper bounds on the number of rows of bounded orthonormal dictionaries (matrix $\mathbf{A}$) for which $\mathbf{A}$ is guaranteed to satisfy the restricted isometry property with high probability. One of the first BOS results was established by~\cite{candes2006near}, where the authors proved an upper bound scales as $\mathcal{O}(sd^6\log^6 n)$ for a subsampled Fourier matrix. While this result is seminal, it is only optimal up to some \textit{polylog} factors. In fact, the authors in chapter $12$ of \cite{foucart2017mathematical} have shown a necessary condition (lower bound) on the number of rows of BOS which scales as $\mathcal{O}(sd\log n)$. In an attempt to achieve to this lower bound, the result in~\cite{candes2006near} was further improved by~\cite{rudelson2008sparse} to $\mathcal{O}(sd\log^2s\log(sd\log n)\log n)$. Motivated by this result, \cite{cheraghchi2013restricted} has even reduced the gap further by proving an upper bound on the number of rows as $\mathcal{O}(sd\log^3 s\log n)$.  The best known available upper bound on the number of rows appears to be $\mathcal{O}(sd^2\log s\log^2n)$; however with worse dependency on the constant $\delta$, i.e., $\delta^{-4}$ (please see~\cite{bourgain2014improved}). To the best of our knowledge, the best known result with mild dependency on $\delta$ (i.e., $\delta^{-2}$) is due to~\cite{haviv2017restricted}, and is given by $\mathcal{O}(sd\log^2s\log n)$. We have used this result for proving Theorem~\ref{thm:recovery on main}. 

\subsection{Training Details on other Datasets}
\label{appendix: training detail}

\subsubsection{CIFAR-100} This dataset is extended version of CIFAR-10 with 100 classes containing 600 images each. Similar to CIFAR-10, CIFAR100 consists of 60,000 color images which splits into 50,000 training images and 10,000 test images. Following~\cite{liu2018darts}, we train the architecture with 20 stacked cells equivalent to CIFAR-10 setting. We train the architecture for 600 epochs with cosine annealing learning rate where the initial value is 0.025. We use a batch size 96, SGD optimizer with nestrov-momentum of 0.9, and auxiliary tower with weights 0.4. For the regularization technique, we include path dropout with probability 0.2, cutout regularizer with length 16, and AutoAugment~\cite{cubuk2018autoaugment} for CIFAR-100. Except AutoAugment, the training setup is identical to DARTs for CIFAR-10.

\subsubsection{Street View House Numbers (SVHN)} SVHN is a digit recognition dataset of house numbers obtained from Google Street View images. SVHN consists of 73,257 train digit images, 26,032 test digit images, and additional 531,131 images. We used both train and extra (total 604,388) images for the training the architecture. Due to the large dataset, we train the architecture for 160 epochs (equivalent to ) and other hyperparameter setup is equivalent to CIFAR-100.

\subsubsection{Fashion-MNIST} Fashion-MNIST consists of 60,000 grayscale training images and 10,000 test images with size $28 \times 28$, classified in 10 classes of objects. Training hyperparameter setup of the final architecture is equivalent to CIFAR-10 without AutoAugment~\cite{cubuk2018autoaugment}.  

We list the network architectures found from our experiments which were not included in the main section.
% % We list all the network architecture found from our experiments including both convolutional and recurrent network.
% \subsection{Compressive sensing-based Neural Architecture Search (CoNAS) for CNN with four stages of sparse recovery}
% \label{appendix:CNN CoNAS t=4}
\begin{figure}[!ht]
    \centering
    \subfloat[Normal Cell]{\includegraphics[width=1.0\linewidth]{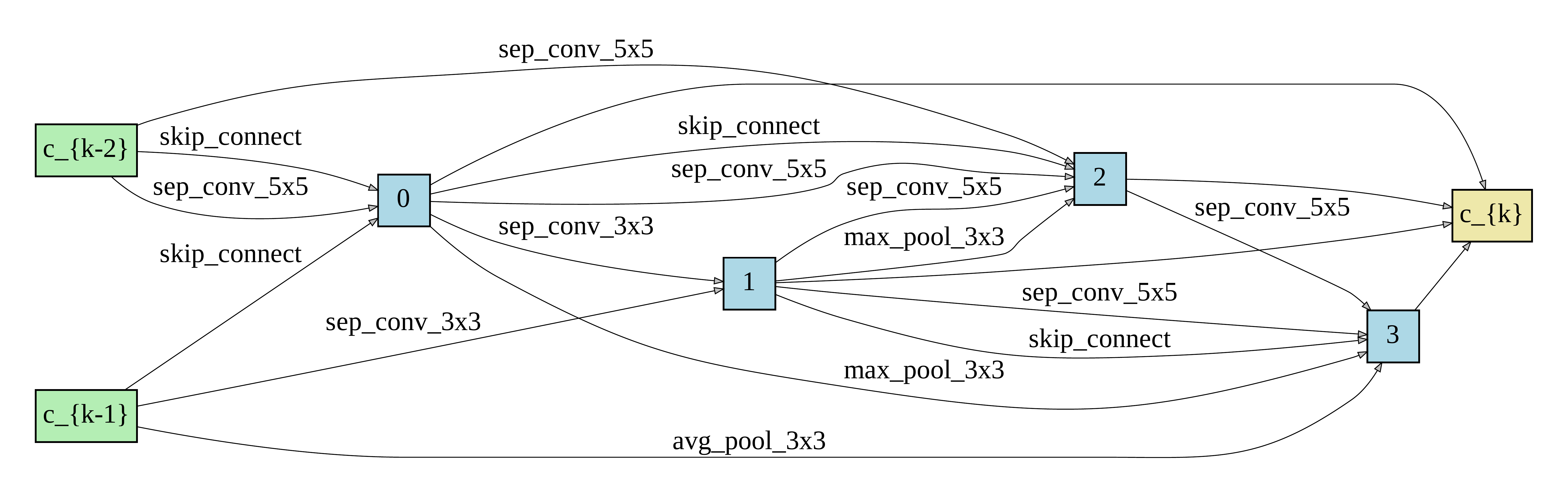}} 
    \vfil
    \subfloat[Reduce Cell]{\includegraphics[width=1.0\linewidth]{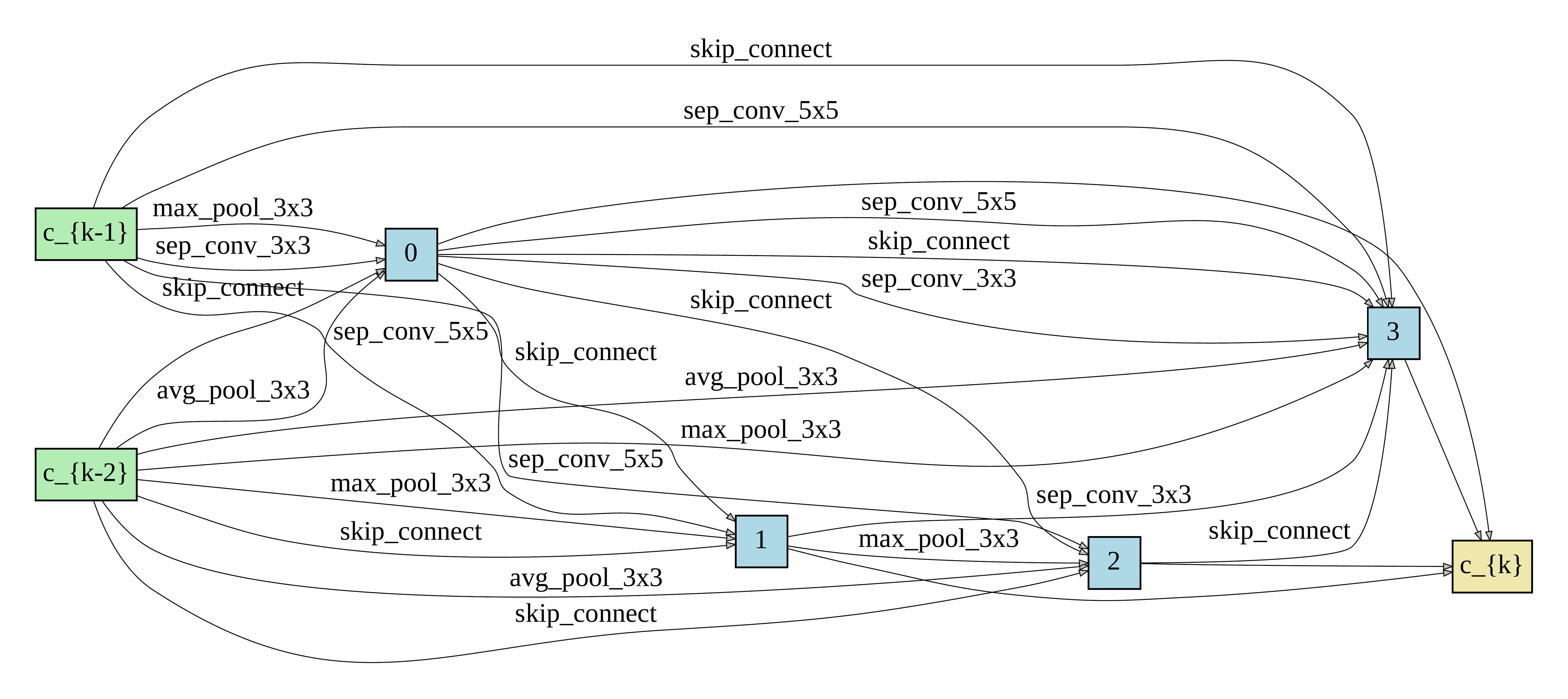}}
    %\vspace{5mm}
    \caption{Convolution Cell found from CoNAS (t=4)}
    \label{fig:cnn cell, t=4}
\end{figure}

% % \subsection{Compressive sensing-based Neural Architecture Search (CoNAS) for RNN}
% \label{appendix:RNN CoNAS}
% \begin{figure}[!ht]
%     \centering
%     \includegraphics[width=1.0\linewidth]{figs/CoNAS_recurrent.pdf}
%     \vspace{5mm}
%     \caption{Recurrent Cell found from CoNAS}
%     \label{fig:rnn cell}
% \end{figure}

% % \subsection{Differentiable Neural Architecture Search (DARTs) for CNN}
\begin{figure}[!ht]
    \centering
    \subfloat[Normal Cell]{\includegraphics[width=0.9\linewidth]{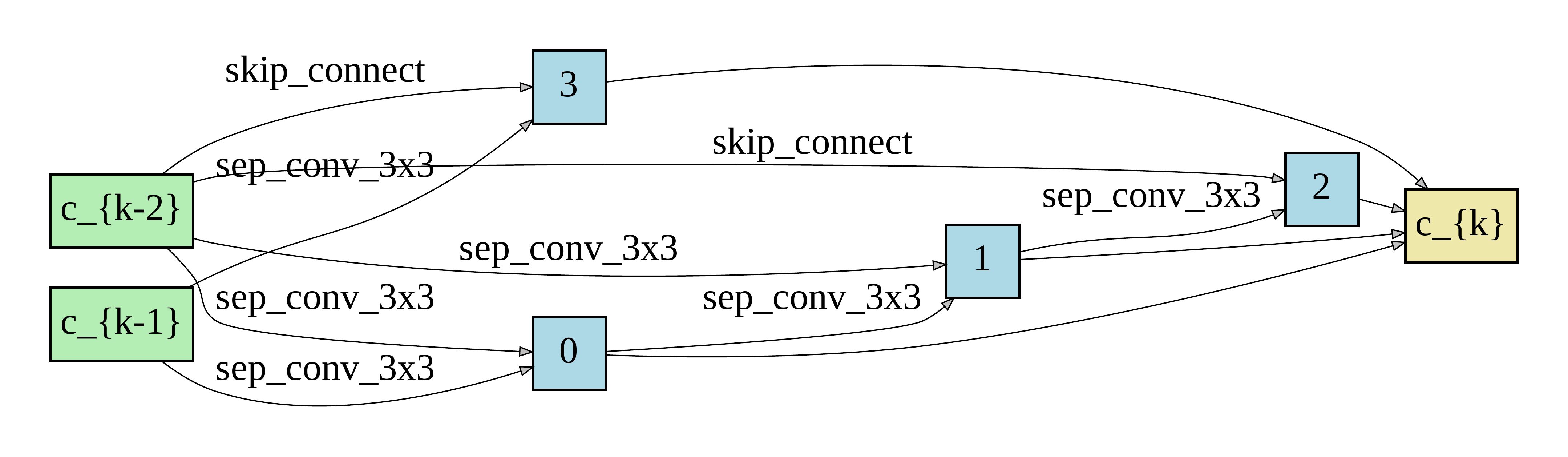}} \\
    \subfloat[Reduce Cell]{\includegraphics[width=0.9\linewidth]{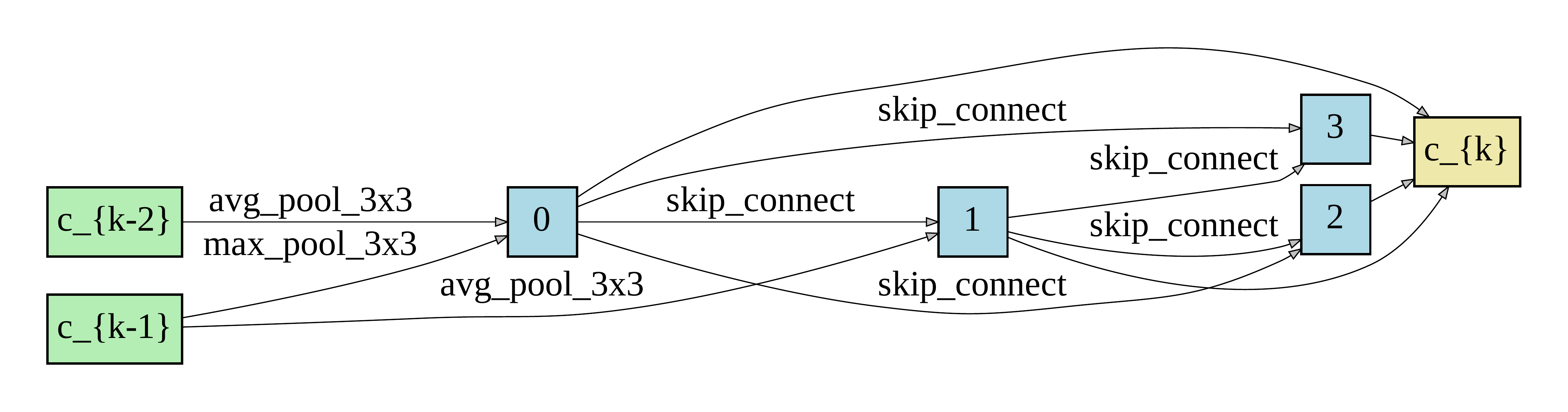}}
    %\vspace{5mm}
    \caption{Convolutional Cell found from DARTs with the original setting in \cite{liu2018darts}.}
\end{figure}

\newpage



% \bibliographystyle{IEEEtran}
\bibliographystyle{amsalpha}
% \bibliography{IEEEabrv,main}
\bibliography{main}

\newcommand{\etalchar}[1]{$^{#1}$}
\providecommand{\bysame}{\leavevmode\hbox to3em{\hrulefill}\thinspace}
\providecommand{\MR}{\relax\ifhmode\unskip\space\fi MR }
\providecommand{\MRhref}[2]{%
  \href{http://www.ams.org/mathscinet-getitem?mr=#1}{#2}
}
\providecommand{\href}[2]{#2}
\begin{thebibliography}{HLVDMW17}

\bibitem[BAPB17]{bertrandhyperparameter}
Hadrien Bertrand, Roberto Ardon, Matthieu Perrot, and Isabelle Bloch,
  \emph{Hyperparameter optimization of deep neural networks: Combining
  hyperband with bayesian model selection}.

\bibitem[BB12]{bergstra2012random}
James Bergstra and Yoshua Bengio, \emph{Random search for hyper-parameter
  optimization}, Journal of Machine Learning Research \textbf{13} (2012),
  no.~Feb, 281--305.

\bibitem[BBBK11]{bergstra2011algorithms}
James~S Bergstra, R{\'e}mi Bardenet, Yoshua Bengio, and Bal{\'a}zs K{\'e}gl,
  \emph{Algorithms for hyper-parameter optimization}, Advances in neural
  information processing systems, 2011, pp.~2546--2554.

\bibitem[BD09]{blumensath2009iterative}
Thomas Blumensath and Mike~E Davies, \emph{Iterative hard thresholding for
  compressed sensing}, Applied and computational harmonic analysis (2009).

\bibitem[Ben00]{bengio2000gradient}
Yoshua Bengio, \emph{Gradient-based optimization of hyperparameters}, Neural
  computation \textbf{12} (2000), no.~8, 1889--1900.

\bibitem[BKZ{\etalchar{+}}18]{bender2018understanding}
Gabriel Bender, Pieter-Jan Kindermans, Barret Zoph, Vijay Vasudevan, and Quoc
  Le, \emph{Understanding and simplifying one-shot architecture search}, Proc.
  Int. Conf. Machine Learning (2018).

\bibitem[Bou14]{bourgain2014improved}
Jean Bourgain, \emph{An improved estimate in the restricted isometry problem},
  Geometric Aspects of Functional Analysis: Israel Seminar (GAFA) (2014), 65.

\bibitem[Can08]{candes2008restricted}
Emmanuel~J Candes, \emph{The restricted isometry property and its implications
  for compressed sensing}, Comptes rendus mathematique \textbf{346} (2008),
  no.~9-10, 589--592.

\bibitem[CCZ{\etalchar{+}}18]{cai2018efficient}
Han Cai, Tianyao Chen, Weinan Zhang, Yong Yu, and Jun Wang, \emph{Efficient
  architecture search by network transformation}, Proc. Assoc. Adv. Art.
  Intell. (AAAI) (2018).

\bibitem[CGV13]{cheraghchi2013restricted}
Mahdi Cheraghchi, Venkatesan Guruswami, and Ameya Velingker, \emph{Restricted
  isometry of fourier matrices and list decodability of random linear codes},
  SIAM Journal on Computing (2013).

\bibitem[CH19]{cho2019reducing}
Minsu Cho and Chinmay Hegde, \emph{Reducing the search space for hyperparameter
  optimization using group sparsity}, Proc. IEEE Int. Conf. Acoust., Speech,
  and Signal Processing (ICASSP) (2019).

\bibitem[CSH19]{cho2019one}
Minsu Cho, Mohammadreza Soltani, and Chinmay Hegde, \emph{One-shot neural
  architecture search via compressive sensing}, arXiv preprint arXiv:1906.02869
  (2019).

\bibitem[CT06]{candes2006near}
Emmanuel~J Candes and Terence Tao, \emph{Near-optimal signal recovery from
  random projections: Universal encoding strategies?}, IEEE Trans. Inform.
  Theory (2006).

\bibitem[CZH19]{cai2018proxylessnas}
Han Cai, Ligeng Zhu, and Song Han, \emph{Proxylessnas: Direct neural
  architecture search on target task and hardware}, Proc. Int. Conf. Learning
  Representations (2019).

\bibitem[CZM{\etalchar{+}}19]{cubuk2018autoaugment}
Ekin~D Cubuk, Barret Zoph, Dandelion Mane, Vijay Vasudevan, and Quoc~V Le,
  \emph{Autoaugment: Learning augmentation policies from data}, {IEEE} Conf.
  Comp. Vision and Pattern Recog (2019).

\bibitem[D{\etalchar{+}}06]{donoho2006compressed}
David~L Donoho et~al., \emph{Compressed sensing}, IEEE Trans. Inform. Theory
  (2006).

\bibitem[EFH{\etalchar{+}}13]{eggensperger2013towards}
Katharina Eggensperger, Matthias Feurer, Frank Hutter, James Bergstra, Jasper
  Snoek, Holger Hoos, and Kevin Leyton-Brown, \emph{Towards an empirical
  foundation for assessing bayesian optimization of hyperparameters}, NIPS
  workshop on Bayesian Optimization in Theory and Practice, vol.~10, 2013,
  p.~3.

\bibitem[EMH18]{elsken2018neural}
Thomas Elsken, Jan~Hendrik Metzen, and Frank Hutter, \emph{Neural architecture
  search: A survey}, arXiv preprint arXiv:1808.05377 (2018).

\bibitem[EMH19]{elsken2018efficient}
\bysame, \emph{Efficient multi-objective neural architecture search via
  lamarckian evolution}, Proc. Int. Conf. Learning Representations (2019).

\bibitem[FDFP17]{franceschi2017forward}
Luca Franceschi, Michele Donini, Paolo Frasconi, and Massimiliano Pontil,
  \emph{Forward and reverse gradient-based hyperparameter optimization}, arXiv
  preprint arXiv:1703.01785 (2017).

\bibitem[FKH18]{falkner2018bohb}
Stefan Falkner, Aaron Klein, and Frank Hutter, \emph{Bohb: Robust and efficient
  hyperparameter optimization at scale}, arXiv preprint arXiv:1807.01774
  (2018).

\bibitem[FLF{\etalchar{+}}16]{fu2016drmad}
Jie Fu, Hongyin Luo, Jiashi Feng, Kian~Hsiang Low, and Tat-Seng Chua,
  \emph{Drmad: Distilling reverse-mode automatic differentiation for optimizing
  hyperparameters of deep neural networks}, arXiv preprint arXiv:1601.00917
  (2016).

\bibitem[FR17]{foucart2017mathematical}
Simon Foucart and Holger Rauhut, \emph{A mathematical introduction to
  compressive sensing}, Bull. Am. Math \textbf{54} (2017), 151--165.

\bibitem[HHLB11]{hutter2011sequential}
Frank Hutter, Holger~H Hoos, and Kevin Leyton-Brown, \emph{Sequential
  model-based optimization for general algorithm configuration}, International
  Conference on Learning and Intelligent Optimization, Springer, 2011,
  pp.~507--523.

\bibitem[HKY17]{hazan2017hyperparameter}
Elad Hazan, Adam Klivans, and Yang Yuan, \emph{Hyperparameter optimization: a
  spectral approach}, arXiv preprint arXiv:1706.00764 (2017).

\bibitem[HLC{\etalchar{+}}19]{hu2019efficient}
Hanzhang Hu, John Langford, Rich Caruana, Saurajit Mukherjee, Eric Horvitz, and
  Debadeepta Dey, \emph{Efficient forward architecture search}, Adv. Neural
  Inf. Proc. Sys. (NeurIPS) (2019).

\bibitem[HLVDMW17]{huang2017densely}
Gao Huang, Zhuang Liu, Laurens Van Der~Maaten, and Kilian~Q Weinberger,
  \emph{Densely connected convolutional networks}, {IEEE} Conf. Comp. Vision
  and Pattern Recog (2017).

\bibitem[HR16]{haviv2016list}
Ishay Haviv and Oded Regev, \emph{The list-decoding size of fourier-sparse
  boolean functions}, ACM Transactions on Computation Theory (TOCT) \textbf{8}
  (2016), no.~3, 10.

\bibitem[HR17]{haviv2017restricted}
\bysame, \emph{The restricted isometry property of subsampled fourier
  matrices}, Geometric Aspects of Functional Analysis (2017), 163--179.

\bibitem[HZRS16]{he2016deep}
Kaiming He, Xiangyu Zhang, Shaoqing Ren, and Jian Sun, \emph{Deep residual
  learning for image recognition}, {IEEE} Conf. Comp. Vision and Pattern Recog
  (2016).

\bibitem[IAFS17]{ilievski2017efficient}
Ilija Ilievski, Taimoor Akhtar, Jiashi Feng, and Christine~Annette Shoemaker,
  \emph{Efficient hyperparameter optimization for deep learning algorithms
  using deterministic rbf surrogates.}, AAAI, 2017, pp.~822--829.

\bibitem[JDO{\etalchar{+}}17]{jaderberg2017population}
Max Jaderberg, Valentin Dalibard, Simon Osindero, Wojciech~M Czarnecki, Jeff
  Donahue, Ali Razavi, Oriol Vinyals, Tim Green, Iain Dunning, Karen Simonyan,
  et~al., \emph{Population based training of neural networks}, arXiv preprint
  arXiv:1711.09846 (2017).

\bibitem[JSH18]{jin2018efficient}
Haifeng Jin, Qingquan Song, and Xia Hu, \emph{Efficient neural architecture
  search with network morphism}, arXiv preprint arXiv:1806.10282 (2018).

\bibitem[JT16]{jamieson2016non}
Kevin Jamieson and Ameet Talwalkar, \emph{Non-stochastic best arm
  identification and hyperparameter optimization}, Artificial Intelligence and
  Statistics, 2016, pp.~240--248.

\bibitem[KDVN18]{kumar2018parallel}
Manoj Kumar, George~E Dahl, Vijay Vasudevan, and Mohammad Norouzi,
  \emph{Parallel architecture and hyperparameter search via successive halving
  and classification}, arXiv preprint arXiv:1805.10255 (2018).

\bibitem[LBGR15]{luketina2015scalable}
Jelena Luketina, Mathias Berglund, Klaus Greff, and Tapani Raiko,
  \emph{Scalable gradient-based tuning of continuous regularization
  hyperparameters}, arXiv preprint arXiv:1511.06727 (2015).

\bibitem[LJD{\etalchar{+}}17]{li2017hyperband}
Lisha Li, Kevin Jamieson, Giulia DeSalvo, Afshin Rostamizadeh, and Ameet
  Talwalkar, \emph{Hyperband: A novel bandit-based approach to hyperparameter
  optimization}, The Journal of Machine Learning Research \textbf{18} (2017),
  no.~1, 6765--6816.

\bibitem[LJR{\etalchar{+}}18]{li2018massively}
Liam Li, Kevin Jamieson, Afshin Rostamizadeh, Ekaterina Gonina, Moritz Hardt,
  Benjamin Recht, and Ameet Talwalkar, \emph{Massively parallel hyperparameter
  tuning}, arXiv preprint arXiv:1810.05934 (2018).

\bibitem[LSV{\etalchar{+}}17]{liu2017hierarchical}
Hanxiao Liu, Karen Simonyan, Oriol Vinyals, Chrisantha Fernando, and Koray
  Kavukcuoglu, \emph{Hierarchical representations for efficient architecture
  search}, arXiv preprint arXiv:1711.00436 (2017).

\bibitem[LSY18]{liu2018darts}
Hanxiao Liu, Karen Simonyan, and Yiming Yang, \emph{Darts: Differentiable
  architecture search}, Proc. Int. Conf. Machine Learning (2018).

\bibitem[LT19]{li2019random}
Liam Li and Ameet Talwalkar, \emph{Random search and reproducibility for neural
  architecture search}, arXiv preprint arXiv:1902.07638 (2019).

\bibitem[LTQ{\etalchar{+}}18]{luo2018neural}
Renqian Luo, Fei Tian, Tao Qin, Enhong Chen, and Tie-Yan Liu, \emph{Neural
  architecture optimization}, Adv. Neural Inf. Proc. Sys. (NeurIPS) (2018).

\bibitem[LZN{\etalchar{+}}18]{liu2018progressive}
Chenxi Liu, Barret Zoph, Maxim Neumann, Jonathon Shlens, Wei Hua, Li-Jia Li,
  Li~Fei-Fei, Alan Yuille, Jonathan Huang, and Kevin Murphy, \emph{Progressive
  neural architecture search}, Euro. Conf. Comp. Vision (2018).

\bibitem[MDA15]{maclaurin2015gradient}
Dougal Maclaurin, David Duvenaud, and Ryan Adams, \emph{Gradient-based
  hyperparameter optimization through reversible learning}, International
  Conference on Machine Learning, 2015, pp.~2113--2122.

\bibitem[NNR{\etalchar{+}}19]{noy2019asap}
Asaf Noy, Niv Nayman, Tal Ridnik, Nadav Zamir, Sivan Doveh, Itamar Friedman,
  Raja Giryes, and Lihi Zelnik-Manor, \emph{Asap: Architecture search, anneal
  and prune}, arXiv preprint arXiv:1904.04123 (2019).

\bibitem[O'D14]{o2014analysis}
Ryan O'Donnell, \emph{Analysis of boolean functions}, Cambridge University
  Press, 2014.

\bibitem[PGZ{\etalchar{+}}18]{pham2018efficient}
Hieu Pham, Melody~Y Guan, Barret Zoph, Quoc~V Le, and Jeff Dean,
  \emph{Efficient neural architecture search via parameter sharing}, Proc. Int.
  Conf. Machine Learning (2018).

\bibitem[RAHL19]{real2018regularized}
Esteban Real, Alok Aggarwal, Yanping Huang, and Quoc~V Le, \emph{Regularized
  evolution for image classifier architecture search}, Proc. Assoc. Adv. Art.
  Intell. (AAAI) (2019).

\bibitem[RV08]{rudelson2008sparse}
Mark Rudelson and Roman Vershynin, \emph{On sparse reconstruction from fourier
  and gaussian measurements}, Communications on Pure and Applied Mathematics: A
  Journal Issued by the Courant Institute of Mathematical Sciences (2008).

\bibitem[SK12]{stobbe2012learning}
Peter Stobbe and Andreas Krause, \emph{Learning fourier sparse set functions},
  Proc. Int. Conf. Art. Intell. Stat. (AISTATS) (2012).

\bibitem[SLA12]{snoek2012practical}
Jasper Snoek, Hugo Larochelle, and Ryan~P Adams, \emph{Practical bayesian
  optimization of machine learning algorithms}, Advances in neural information
  processing systems, 2012, pp.~2951--2959.

\bibitem[SLJ{\etalchar{+}}15]{szegedy2015going}
Christian Szegedy, Wei Liu, Yangqing Jia, Pierre Sermanet, Scott Reed, Dragomir
  Anguelov, Dumitru Erhan, Vincent Vanhoucke, and Andrew Rabinovich,
  \emph{Going deeper with convolutions}, {IEEE} Conf. Comp. Vision and Pattern
  Recog (2015).

\bibitem[SSZA14]{snoek2014input}
Jasper Snoek, Kevin Swersky, Rich Zemel, and Ryan Adams, \emph{Input warping
  for bayesian optimization of non-stationary functions}, Proc. Int. Conf.
  Machine Learning, 2014, pp.~1674--1682.

\bibitem[SYJ{\etalchar{+}}19]{sciuto2019evaluating}
Christian Sciuto, Kaicheng Yu, Martin Jaggi, Claudiu Musat, and Mathieu
  Salzmann, \emph{Evaluating the search phase of neural architecture search},
  arXiv preprint arXiv:1902.08142 (2019).

\bibitem[THHLB13]{thornton2013auto}
Chris Thornton, Frank Hutter, Holger~H Hoos, and Kevin Leyton-Brown,
  \emph{Auto-weka: Combined selection and hyperparameter optimization of
  classification algorithms}, Proceedings of the 19th ACM SIGKDD international
  conference on Knowledge discovery and data mining, ACM, 2013, pp.~847--855.

\bibitem[Tib96]{tibshirani1996regression}
Robert Tibshirani, \emph{Regression shrinkage and selection via the lasso},
  Journal of the Royal Statistical Society. Series B (Methodological) (1996),
  267--288.

\bibitem[WXW18]{wang2018combination}
Jiazhuo Wang, Jason Xu, and Xuejun Wang, \emph{Combination of hyperband and
  bayesian optimization for hyperparameter optimization in deep learning},
  arXiv preprint arXiv:1801.01596 (2018).

\bibitem[XKGH19]{xie2019exploring}
Saining Xie, Alexander Kirillov, Ross Girshick, and Kaiming He, \emph{Exploring
  randomly wired neural networks for image recognition}, arXiv preprint
  arXiv:1904.01569 (2019).

\bibitem[XZLL18]{xie2018snas}
Sirui Xie, Hehui Zheng, Chunxiao Liu, and Liang Lin, \emph{Snas: stochastic
  neural architecture search}, arXiv preprint arXiv:1812.09926 (2018).

\bibitem[YIAK18]{yamada2018shakedrop}
Yoshihiro Yamada, Masakazu Iwamura, Takuya Akiba, and Koichi Kise,
  \emph{Shakedrop regularization for deep residual learning}, Proc. Int. Conf.
  Learning Representations Workshop (2018).

\bibitem[YL06]{yuan2006model}
Ming Yuan and Yi~Lin, \emph{Model selection and estimation in regression with
  grouped variables}, Journal of the Royal Statistical Society: Series B
  (Statistical Methodology) \textbf{68} (2006), no.~1, 49--67.

\bibitem[ZL17]{zoph2016neural}
Barret Zoph and Quoc~V Le, \emph{Neural architecture search with reinforcement
  learning}, Proc. Int. Conf. Learning Representations (2017).

\bibitem[ZRU18]{zhang2018graph}
Chris Zhang, Mengye Ren, and Raquel Urtasun, \emph{Graph hypernetworks for
  neural architecture search}, Proc. Int. Conf. Learning Representations
  (2018).

\bibitem[ZVSL18]{zoph2018learning}
Barret Zoph, Vijay Vasudevan, Jonathon Shlens, and Quoc~V Le, \emph{Learning
  transferable architectures for scalable image recognition}, {IEEE} Conf.
  Comp. Vision and Pattern Recog (2018).

\end{thebibliography}
\end{document}